\documentclass[twoside]{article}

\usepackage{natbib}
\usepackage[utf8]{inputenc} % allow utf-8 input
\usepackage[T1]{fontenc}    % use 8-bit T1 fonts
\usepackage{hyperref}
\hypersetup{colorlinks, citecolor=black, linkcolor=black, urlcolor=blue}
\usepackage{url}            % simple URL typesetting
\usepackage{booktabs}       % professional-quality tables
\usepackage{amsfonts}       % blackboard math symbols
\usepackage{nicefrac}       % compact symbols for 1/2, etc.
\usepackage{microtype}      % microtypography
\usepackage[x11names, dvipsnames]{xcolor}

\usepackage{graphicx}
\usepackage{caption}
\usepackage{subcaption}
\usepackage{algorithm}
\usepackage{algorithmic}
\usepackage{mathtools}
\usepackage{multirow}

\newcommand{\mD}{\mathcal{D}}

\usepackage{soul}
\usepackage{xcolor}

\newcommand{\Mean}{{\mathbb{E}}}

\newcommand{\prob}{\mathbb{P}}

\usepackage{amsthm}

\usepackage{array}
\usepackage{amsmath}
\usepackage{amssymb}
\usepackage{bm}
\DeclareMathOperator*{\argmin}{arg\;\!min}
\DeclareMathOperator*{\argmax}{arg\;\!max}

\newcommand{\data}{\mathcal{D}}
\newcommand{\median}{\operatorname{Median}}
\newcommand{\name}[1]{\texttt{#1}}

\newtheorem{theorem}{Theorem}
\newtheorem{proposition}{Proposition}
\newtheorem{lemma}{Lemma}

\newtheorem{definition}{Definition}

\usepackage[accepted]{aistats2024}
\begin{document}

\runningtitle{Robust Offline Policy Evaluation and Optimization with Heavy-Tailed Rewards}


\twocolumn[

\aistatstitle{Robust Offline Reinforcement Learning with Heavy-Tailed Rewards}

\aistatsauthor{ Jin Zhu \And Runzhe Wan \And Zhengling Qi \And Shikai Luo \And Chengchun Shi}

\aistatsaddress{ LSE \And  Amazon \And George Washington University \And Bytedance \And LSE} ]

\begin{abstract}
    This paper endeavors to augment the robustness of offline reinforcement learning (RL) in scenarios laden with heavy-tailed rewards, a prevalent circumstance in real-world applications. We propose two algorithmic frameworks, \name{ROAM} and \name{ROOM}, for robust off-policy evaluation and offline policy optimization (OPO), respectively. Central to our frameworks is the strategic incorporation of the median-of-means method with offline RL, enabling straightforward uncertainty estimation for the value function estimator. This not only adheres to the principle of pessimism in OPO but also adeptly manages heavy-tailed rewards. Theoretical results and extensive experiments demonstrate that our two frameworks outperform existing methods on the logged dataset exhibits heavy-tailed reward distributions. The implementation of the proposal is available at \url{https://github.com/Mamba413/ROOM}.
\end{abstract}

\graphicspath{{./figure/}{./figure/opo-result/}{./figure/ope-result/}}

\section{INTRODUCTION}

In reinforcement learning \citep[RL,][]{sutton2018reinforcement}, evaluating and optimizing policies without accessing the environment becomes crucial nowadays, because frequently interacting with the environment could be prohibitively expensive or even impractical in many real-world applications such as robotics, healthcare, education, autonomous driving, and so on. This leads to a surge of interest in offline RL \citep{levine2020offline,uehara2022review}, which aims to leverage only logged data for policy evaluation and optimization. 

% \red{I guess you want to say: (1) Heavy-tailed reward is commonly seen in practice such as finance; (2) The existence of heavy-tailed significantly degrades the performance of offline RL; Can you make these two points more concisely?}

The success of offline RL so far crucially relies on that the reward distribution is well-behaved. However, in a number of real-world applications, the reward distribution is usually heavy-tailed\footnote{A random variable is called heavy-tailed when its tail distribution is heavier than the exponential distribution, and sometimes even its variance is not well defined.}. Heavy-tailed rewards can be generated by various real-world decision-making systems, such as the stock market, networking routing, scheduling, hydrology, image, audio, and localization errors, etc \citep{georgiou1999alpha, hamza2001image, huang2017new, ruotsalainen2018error}. 

The heavy-tailedness pose great challenges to existing offline RL methods. 
% We illustrate this via two fundamental problems in offline RL: off-policy evaluation (OPE) and offline policy optimization (OPO). OPE/OPO aims to evaluate/maximize the value function of policies using only historical data without further interactions with the environment. OPE and OPO can be handled by fitted Q-evaluation (FQE) and fitted Q-iteration (FQI) algorithm, both of them are one of the most popular iterative algorithms in RL.\red{(to smooth.)} Yet, since the heavy-tailed rewards have a huge variance, the estimation for the value of Q-function would becomes inaccuracy, accompanying with an algorithmic instability. And thus, the heavy-tailed rewards definitely degrade the performance of FQE and FQI. 
We first illustrate this via a fundamental problem in offline RL: off-policy evaluation (OPE). OPE aims to evaluate the value of policies using only logged data. One classic algorithm is fitted-Q evaluation (FQE), where each step is solving a regression problem with the response variable being the observed reward plus some estimated long-term values. Yet, it is well-known that the performance of standard regression methods is very sensitive to heavy-tailed responses \citep{lugosi2019mean} and will have a much slower convergence rate. Consequently, this issue will degrade the performance of policy evaluation. 
% since the heavy-tailed rewards have a huge variance, the estimation for temporal difference would become inaccuracy. And thus, the heavy-tailed rewards definitely degrade the performance of FQE. 

As for offline policy optimization (OPO), the heavy-tailed rewards pose even more challenges because the issue of overestimation in standard RL algorithms could be aggravated. We elaborate this with a bandit example shown in Figure~\ref{fig:bandit}, a special case of RL. In this example, the large variance in estimating the expected reward causes a non-negligible probability of selecting the sub-optimal arm. In settings with heavy-tailed rewards, the empirical mean of the sub-optimal arm is subject to an even larger variance, leading to a higher probability of selecting the sub-optimal arm.

\begin{figure}[htbp]
\begin{center}
\includegraphics[width=1.05\linewidth]{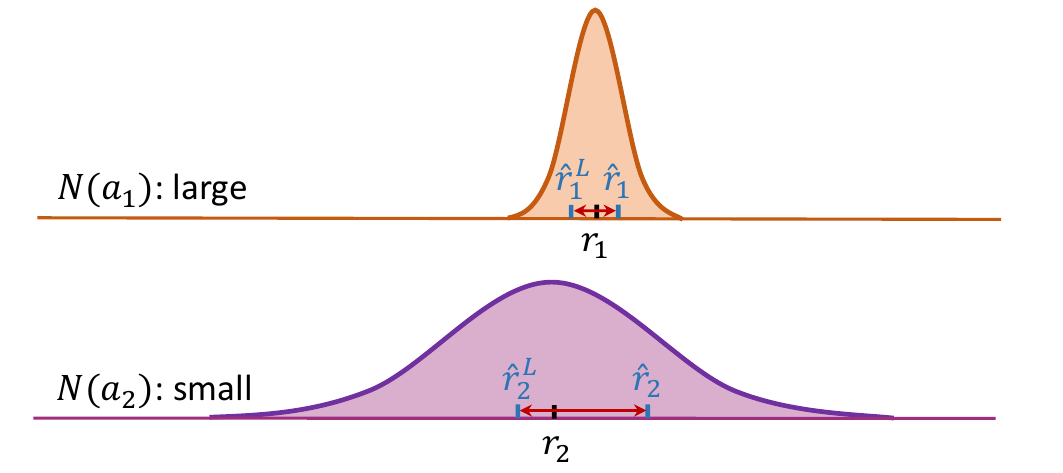}
\end{center}
\caption{Reward distributions in a two-armed bandit example. The oracle expected rewards for the two arms $a_k$ are given by $r_k$ (for $k =1, 2$). $N(a_k)$ denotes the number of reward observations for the $k$-arm. The expected rewards estimator is given by $\widehat{r}_k$. Due to the limited sample size for the second sub-optimal arm, its estimated expected reward $\widehat{r}_2$ suffers from a large variance. Consequently, there's a non-negligible probability of $\widehat{r}_2 > \widehat{r}_1$. By penalizing the uncertainty of reward estimation, a pessimistic estimation $\widehat{r}^L_k$ lowers bound the reward, leading to $\widehat{r}_2^L<\widehat{r}_1^L$, yielding the optimal action.}
\label{fig:bandit}
\end{figure}

To accommodate the heavy-tailed rewards in offline RL, we propose new frameworks for both OPE and OPO by leveraging the median-of-means (MM) estimator in robust statistics \citep{nemirovskij1983problem, alon1996space}. Specifically, we design frameworks that can effectively robustify existing RL algorithms against the heavy-tailed rewards. The frameworks are simple and easy-to-implement. Moreover, the proposed approach also provides a natural way for qualifying the uncertainty of value estimation, which is crucial in both OPE and OPO. 

% \textbf{Contribution.}
\subsection{Contribution}
The contribution of this paper is three-fold. First, we propose a general and unified framework to improve the robustness of existing OPE and OPO methods against heavy-tailed rewards. By leveraging MM, our approach naturally allows uncertainty quantification of the estimated values. This is critical for OPE in high-risk applications (e.g., healthcare) and also for OPO to incorporate the principle of pessimism \citep{jin2021pessimism, bai2022pessimistic} to address insufficent data coverage. 

Second, we provide rigorous theoretical analyses on our OPE and OPO algorithms and clearly demonstrate their advantages over existing solutions that overlook the heavy-tailed issue. In particular, our analysis only requires the reward to have finite $(1+\alpha)$-th moment. On the contrary, most of the existing methods require the rewards to be bounded (or sub-Gaussian) to achieve similar performance. 

Finally, on a couple of benchmark OpenAI environments, we observe the superiority of the proposed algorithms against existing ones when the rewards are heavy-tailed. 
In particular, for OPE, our methods are 1.5 to 30 times more accurate than the non-robust algorithms in terms of rooted MSE; on several D4RL benchmarks for OPO, the score of the robust version is about 1.3 to 3 times higher than those of the vanilla version of state-of-the-art (SOTA) algorithm. 

% to study the performance of our proposal on offline datasets with  rewards, in 
%\item First, we propose general OPE and OPO algorithms for addressing the heavy-tailed rewards observed in the offline dataset. 
% To the best of our knowledge, the two algorithms fill the blank of robust .
%\item Second

\section{RELATED WORKS}

\textbf{Off-Policy Evaluation.}
% \subsection{Off-Policy Evaluation}
In the literature, there are three commonly-used approaches for OPE. 
The first one is the direct method (DM), which evaluates the target policy by estimating its Q-function \citep{bertsekas2012dynamic, farajtabar2018more, le2019batch, duan2020minimax, luckett2020estimating, liao2021off,hao2021bootstrapping,shi2022statistical}. 
Importance sampling (IS) is another popular OPE approach \citep{precup2000eligibility, thomas2015high, liu2018breaking, pmlr-v97-hanna19a, nachum2019dualdice, xie2019towards, dai2020coindice,wang2021projected}, motivated by the change of measure theorem. 
Sequential IS gives an unbiased estimator but has an exponentially large variance with respect to the horizon. 
\citet{liu2018breaking, xie2019towards} developed marginal IS estimators to break this curse of horizon. 
The last approach aims to exploit the advantages of both DM and IS, by combining them to derive a doubly robust (DR) estimator \citep{thomas2016data,jiang2016doubly, farajtabar2018more, tang2019doubly, kallus2020double, liao2022batch}. 
We refer to \citet{uehara2022review} for a comprehensive review for OPE. 
However, to our knowledge, most existing methods cannot handle the heavy-tailed rewards. 

\textbf{Offline Policy Optimization.}
% \subsection{Offline Policy Optimization}
It is well-known that standard OPO methods \citep[e.g.,][]{ernst2005tree} may fail to converge and produce unstable solutions due to the distributional mismatch in the offline setting \citep{wang2021instabilities}. To address this limitation, one possible approach is to force the learned policy to choose actions close to the observed ones in the offline data  \citep{wu2020behavior, brandfonbrener2021offline, fujimoto2021minimalist, kostrikov2021offline, dadashi2021offline}. Recently, there is a streamline of research utilizing the principle of pessimism to address the insufficient data coverage issue \citep[e.g.,][]{kumar2020conservative, an2021uncertainty, jin2021pessimism, xie2021bellman, yin2021nearoptimal, yu2021combo, bai2022pessimistic, fu2022offline, kostrikov2022offline, guo2022model, shi2022pessimistic, uehara2022pessimistic, lyu2022mildly, fu2022closer, zhou2022optimizing, xu2023offline, zhang2023insample, zhou2023bilevel, chen2023steel}. We refer interested readers to \citet{prudencio2022survey} for a recent survey on OPO. However, existing OPO methods cannot handle heavy-tailed rewards. In addition, recent OPO methods proposed to use confidence intervals (CIs) to quantify the uncertainty of the estimated Q-function \citep{jin2021pessimism,bai2022pessimistic}. These CIs could be unreasonably wide due to the heavy-tailed rewards.

% \subsection{Robust RL}
\textbf{Robust RL.}
Most existing works on handling heavy-tailed rewards are only designed for bandits, a special case of RL. Various robust mean estimators are proposed for designing algorithms in finding an optimal arm in the online setting \citep[e.g., ][]{bubeck2013bandits,shao2018almost,lu2019optimal,zhong2021breaking}. However, less attention has been paid to heavy-tailed rewards when there has a state transition. To the best of our knowledge, \citet{zhuang2021no, pmlr-v202-rowland23a, liu2023online} are the most related papers studying this issue. They focus on an online setting which is substantially different from our offline setting. 

We remark that there has another line of research on robust offline RL \citep{chen2021improved, lykouris2021corruption, mo2021learning, si2020distributionally, zhang2021robust, kallus2022doubly, xu2022quantile, zhang2022corruption}, which mainly focuses on robust decision making under the uncertainty of the changing environment. Another stream studies OPE/OPO under the robust Markov decision process \citep{nilim2005robust} by exploiting prior distributional information allow uncertainty quantification  \citep{mannor2016robust, wiesemann2013robust, wang2022reliable, goyal2023robust}. In summary, the goals of these research are different from ours. 

\section{PRELIMINARIES}
% \subsection{Markov decision process}
\textbf{Markov decision process.}
We consider an infinite-horizon discounted stationary Markov Decision Process \citep[MDP,][]{sutton2018reinforcement}, which is defined by a tuple $\mathcal{M} \coloneqq (\mathcal{S}, \mathcal{A}, \mathcal{P}, \mathcal{R}, \gamma)$ where $\mathcal{S}$ is the state space, $\mathcal{A}$ is the action space, 
the transition kernel $\mathcal{P}(\bullet| S_{t-1}, A_{t-1})$ specifies the probability mass (or density) function of $S_t$ by taking action $A_{t-1}$ at a state $S_{t-1}$, and similarly $\mathcal{R}$ specifies the reward. The constant $\gamma \in [0,1)$ is the discount factor. 
We denote the initial state distribution as $\mathbb{G}$. 
For simplicity of notations, we assume $\mathbb{G}$ is pre-specified in this paper.  
$\mathbb{G}$ can be estimated from the empirical initial state distribution in practice. 

In the existing literature, the reward is  assumed to be uniformly bounded or at least sub-Gaussian \citep{thomas2016data,fan2020theoretical,chen2022well,shi2023value}. 
However, as discussed in the introduction, such an assumption could be violated in many real applications. 
In this paper, we consider a much milder assumption, that is, the reward distribution  $\mathcal{R}$ has finite $(1+\alpha)$-th moments for some $\alpha \in (0, 1]$. 
Then, the mean reward function $r(s,a) = \mathbb{E}(R_t | S_t = s, A_t = a)$ exists. 
%But other than that, the reward distribution can be arbitrarily heavy-tailed. 
No other assumptions are imposed and the reward distribution can be arbitrarily heavy-tailed. 
Let $\pi(a|s): \mathcal{S} \rightarrow \mathcal{A}$ be a given policy that specifies the conditional distribution of the action given the state. We next the value function  $V^\pi(s)\coloneqq \mathbb{E}^\pi \left[\sum_{t=0}^{\infty} \gamma^tR_t \mid S_0 = s\right]$ and the Q-function $Q^\pi(s, a)\coloneqq \mathbb{E}^\pi \left[\sum_{t=0}^{\infty} \gamma^tR_t \mid S_0 = s, A_0=  a \right]$. 
Let $\mathbb{E}_{\data}[\cdot]$ denote the expectation taken with respect to the empirical measure over the offline data  $\data$.

% Denote the $i$th trajectory as $\tau_i$. 

% For a policy $\pi$, the Bellman operator $\mathcal{T}^\pi$ is defined as $\left(\mathcal{T}^\pi f\right)(s, a):=R(s, a)+\gamma \mathbb{E}_{s^{\prime} \mid s, a}\left[f\left(s^{\prime}, \pi\right)\right]$, where $f\left(s^{\prime}, \pi\right):=\sum_a \pi\left(a^{\prime} \mid s^{\prime}\right) f\left(s^{\prime}, a^{\prime}\right)$. In addition, we use $d^\pi$ to denote the normalized and discounted state-action occupancy measure of the policy $\pi$. That is, $d^\pi(s, a):=$ $(1-\gamma) \mathbb{E}\left[\sum_{t=0}^{\infty} \gamma^t \mathbb{1}\left(s_t=s, a_t=a\right) \mid a_t \sim \pi\left(\cdot \mid s_t\right)\right]$. We also use $\mathbb{E}_\pi$ to denote expectations with respect to $d^\pi$.

% and $J(\pi):=\mathbb{E}\left[\sum_{t=0}^{\infty} \gamma^t r_t \mid a_t \sim \pi\left(\cdot \mid s_t\right)\right]$ to denote the expected discounted return of $\pi$, with $r_t=R\left(s_t, a_t\right)$. The goal of RL is to find a policy that maximizes $J(\cdot)$. For any policy $\pi$, 

% \subsection{Problem Formulation}
\textbf{Problem Formulation.}
We assume that an agent interacts with the environment $\mathcal M$ and collects a series of random tuples in the form of $(S, A, R, S')$ using one behavior policy. The offline dataset $\mD$ consists of all tuples with form $(S,A,R,S')$. There are two main tasks in offline RL as follows.
\begin{itemize}
    \item Off-policy evaluation (OPE): given the offline dataset $\mD$ and a given target policy $\pi$, OPE estimates its value $J^{\pi} \coloneqq \Mean_{s \sim \mathbb{G}, a \sim \pi(\bullet|s)} Q^{\pi}(s, a)$.
    % \begin{eqnarray*}
    % % \label{eqn:def_value}
    %     J^{\pi} = \Mean_{s \sim \mathbb{G}} V^{\pi}(s). 
    % \end{eqnarray*} 
    \item Offline policy optimization (OPO): given the offline dataset $\mD$, OPO aims to learn an optimal  policy ${\pi}^* = \arg \max_{\pi} J^{\pi}$.  
    % based on $\mathcal{D}$ without deploy $\widehat{\pi}^*$ into the MDP $\mathcal{M}$.
\end{itemize}
Most existing methods for the two tasks crucially rely on the assumption that the rewards are uniformly bounded, yet simply employing them cannot address the challenges posed by heavy-tailed rewards, as illustrated in the example below.

\textbf{Failure of standard direct methods.} To illustrate,
we mainly focus on DM for OPE, which has shown promising performances from  theory and empirical studies \citep{duan2020minimax, voloshin2019empirical}. A DM-type OPE algorithm first estimates the Q-function as $\widehat{Q}$ and then estimate the value of $\eta^\pi$ by constructing a plug-in estimator for $\widehat{J}^\pi = \mathbb{E}_{s\sim\mathbb{G}, a \sim \pi(\bullet|s)}\widehat{Q}^\pi(s, a)$. 

To see the failure of DM, we first present the connection between $\mathrm{Q}$-function estimation and population mean estimation. Define the conditional discounted state-action visitation distribution of the tuple $(s, a)$ following policy $\pi$ starting from $(s_0, a_0)$ as $d^\pi\left(a, s | a_0, s_0\right)=(1-\gamma)\{\mathbb{I}(a=a_0, s=s_0)+\sum_{t=1}^{\infty} \gamma^t p_t^\pi\left(a, s \mid a_0, s_0\right)\}$. Then,
\begin{align*}
\begin{aligned}
& Q^\pi(s_0, a_0) \\
% & =\mathbb{E}^\pi\left[\sum_{t=0}^{\infty} \gamma^t R_t \mid S_0=s, A_0=a\right] \\
= &(1-\gamma)^{-1} \mathbb{E}_{\left(S_t, A_t\right) \sim d^\pi\left(a, s \mid a_0, s_0\right), R_t \sim \mathcal{R}\left(S_t, A_t\right)} R_t .
\end{aligned}
\end{align*}
% \begin{align*}
% \begin{aligned}
% Q^\pi\left(s_0, a_0\right) =\mathbb{E}^\pi\left[\sum_{t=0}^{\infty} \gamma^t R_t \mid S_0=s, A_0=a\right]
% =(1-\gamma)^{-1} \mathbb{E}_{\left(S_t, A_t\right) \sim d^\pi\left(a, s \mid a_0, s_0\right), R_t \sim \mathcal{R}\left(S_t, A_t\right)} R_t.
% \end{aligned}
% \end{align*}
In other words, the Q-value $Q^\pi\left(s_0, a_0\right)$ is the population mean of the stochastic rewards under the corresponding state-action visitation distribution induced by policy $\pi$ starting from $\left(s_0, a_0\right)$. The heavy tailedness\footnote{The heavy-tailedness can be caused by either the heavy-tailedness of $\mathcal{R}$ (i.e., $R_t-$ $r\left(S_t, A_t\right)$, the randomness of the stochastic rewards) or that of $r\left(S_t, A_t\right)$ (i.e., the distribution of the mean reward following some policy). Almost all of our discussions can accommodate both sources simultaneously.} of $R_t$ typically will carry over to the distribution of $\sum_{t=0}^{\infty} \gamma^t R_t$ conditioned on $\left\{S_0=s, A_0=a\right\}$ and following $\pi$. In this sense, the estimation of $Q^\pi\left(s_0, a_0\right)$ will face the same challenge as the population mean estimation with heavy-tailed noises\footnote{One may also refer to Theorem~4 in \citet{gerstenberg2022solutions} for a sufficient condition for the cumulative reward to be heavy-tailed.}. At this case, the estimation error of the sample mean $\bar{R}$ can be upper bounded by $|\bar{R}-\mu| \leq C^{\prime} \delta^{-\frac{1}{1+\alpha}} n^{-\frac{\alpha}{1+\alpha}}$ with probability at least $1-\delta$, for a constant $C^{\prime}>0$. To ensure a high-probability result, $\delta$ shall inversely scale polynomially in $n$, causing the error bound may scale polynomially in $n$, which dominate the (constant) reward means.
% For example, when estimating the Q-function via the well-known least-square temporal-difference (LSTD, \citet{bradtke1996linear, boyan1999least}) algorithm, the convergence rate will have a polynomial order dependency on the confidence level instead of a desired logarithmic order (See Section~\ref{sec:theory} for more details).

% \begin{remark}
% The heavy-tailedness of $\sum_{t=0}^{\infty} \gamma^t R_t$ can be caused by either the heavy-tailedness of $\mathcal{R}$ (i.e., $R_t-$ $r\left(S_t, A_t\right)$, the randomness of the stochastic rewards) or that of $r\left(S_t, A_t\right)$ (i.e., the distribution of the mean reward following some policy). Almost all of our discussions can accommodate both sources simultaneously. 
% \end{remark}

\textbf{Median-of-mean method.}
The key tool in our algorithms is the MM estimator \citep{nemirovskij1983problem, alon1996space} in robust statistics. Due to its flexibility and that it is straightforward to produce uncertainty quantification, MM has also been employed in robust linear regressions as well \citep{zhang2021median, minsker2015geometric}. We present its form in population mean estimation and related property below.

\begin{definition}[Population mean estimation via MM]\label{prop:mm-estimator}
    Let $R_1, \ldots, R_n$ be $n$ i.i.d. real-valued heavy-tailed observations under a distribution $F$. To estimate the population mean, we first partition $[n]=\{1, \ldots, n\}$ into $K \in \mathbb{N}^{+}$blocks $B_1, \ldots, B_K$, each of size $\left|B_i\right| \geq\lfloor n / K\rfloor \geq 2$. We compute the sample mean in each block as $Z_k=\frac{1}{\left|B_k\right|} \sum_{i \in B_k} R_i$. The MM estimator for the mean value of $F$ is defined as $\operatorname{Median}\left(\left\{Z_1, \ldots, Z_K\right\}\right)$.
\end{definition}

\begin{proposition}[\citet{lugosi2019mean}, Theorem 3]
    Suppose $R_1, \ldots, R_n$ are i.i.d. with mean $\mu$ and the $(1+\alpha)$ th moment. For any $\delta \in(0,1)$, by setting $K=\lceil 8 \log (2 / \delta)\rceil$, we have with probability at least $1-\delta$ that
    \begin{align*}
    \left|\widehat{\mu}_n-\mu\right| \leq C[\log (1 / \delta)]^{\frac{\alpha}{1+\alpha}} n^{-\frac{\alpha}{1+\alpha}}
    \end{align*}
    for some constant $C>0$.
\end{proposition}

Comparing sample mean and the MM, we easily see that sample mean's dependence on the confidence parameter $\delta$ is exponentially worse than that of MM. Indeed, a sub-Gaussian assumption is typically required for sample mean to enjoy the same property as $\mathrm{MM}$ estimator. Based on the aforementioned discussion, we will borrow ideas from the MM to improve the robustness of OPE and OPO.

% \vspace*{-8pt}
\section{MM FOR ROBUST OFFLINE RL}
In this section, we start by presenting our proposal for OPE to illustrate the main idea of utilizing MM in offline RL to address the heavy-tailed rewards. We then extend the idea to OPO in Section~\ref{sec:mm-room}. 

% {\color{blue}Our objective is to provide a generally applicable recipe to robustify the existing OPE and OPO methods in the presence of heavy-tailed rewards.}

\subsection{MM for OPE}\label{sec:mm-roam}
This section introduces the Robust Off-policy evaluAtion via Median-of-means (\name{ROAM}) framework. For ease of exposition, we first focus on the DM in this paper. Specifically, the discussions above motivate us to consider leveraging the MM scheme for robust estimation of $Q^\pi$ (see an illustration in Figure~\ref{fig:procedure}). 
% We will cover extensions to IS and DR as well. 
% We summarize the first proposed algorithm framework in Algorithm~\ref{alg:mm-final}, which we refer to as \name{ROAM}-based Direct Method (\name{ROAM-DM}) and detail as follows. 
We first split all trajectories $\mathcal{D}$ into $K$ partitions $\left\{\mathcal{D}_k\right\}_{k=1}^K$. %We make sure that one trajectory only belongs to one partition, so that the 
Notice that data subsets across the $K$ splits are \textit{i.i.d}. Next, with any given DM-type OPE algorithm \name{BaseOPE}, we obtain $K$ \textit{i.i.d.} estimates $\{\widehat{Q}_k^\pi\}_{k=1}^K$ for $Q^\pi$. However, with heavy-tailed rewards, these estimates may also have large errors and distributed with heavy tails around $Q^\pi$. As discussed above, this is similar to the sample average in every split for population mean estimation. Therefore, we propose to extend MM to OPE by first applying the median operator to the $K$ Q-functions and then calculate the integrated value estimate as $\widehat{J}^\pi=\mathbb{E}_{s \sim \mathbb{G}, a \sim \pi(\bullet|s)} \operatorname{Median}(\{\widehat{Q}_k^\pi(s, a)\}_{k=1}^K)$. We summarize the procedure in Algorithm~\ref{alg:mm-final}. 
Notably, our approach employs a split-and-aggregation strategy to estimate a robust Q function, which is markedly different from the standard MM method for estimating a scalar. As such, verifying the robustness of the estimated Q function necessitates a non-trivial analysis of the proposed procedure.

\begin{figure}[!t]
    \centering
    \includegraphics[width=1.0\linewidth]{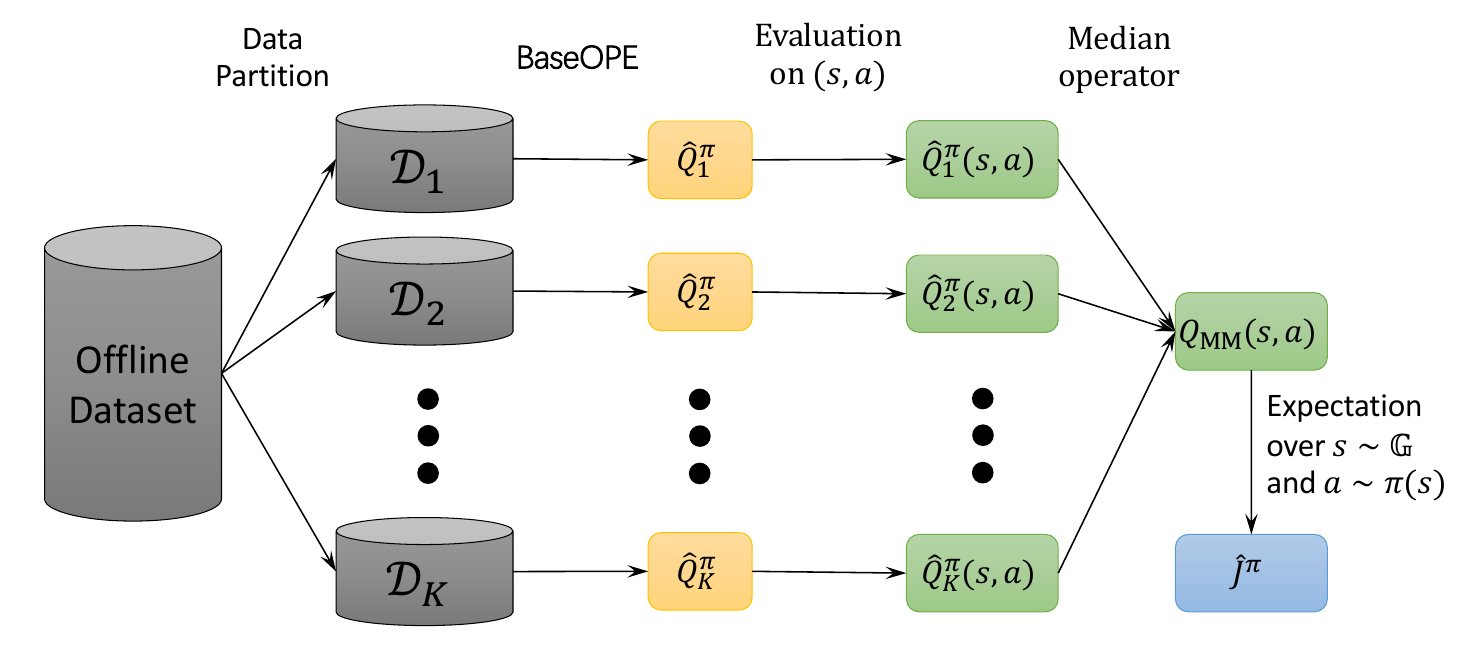}
    \caption{Graphical illustration for \name{ROAM}. $Q_{\mathrm{MM}}(s, a)$ is equal to $\operatorname{Median}(\{\widehat{Q}_k^\pi(s, a)\}_{k=1}^K)$.}
    \label{fig:procedure}
\end{figure}
% \begin{wrapfigure}{r}{0.55\textwidth}
%     \vspace*{-10pt}
%     \centering
%     \includegraphics[width=1.0\linewidth]{procedure.pdf}
%     \vspace*{-13pt}
%     \caption{Graphical illustration for \name{ROAM}. $Q_{\mathrm{MM}}(s, a)$ is equal to $\operatorname{Median}(\{\widehat{Q}_k^\pi(s, a)\}_{k=1}^K)$.}
%     \label{fig:procedure}
% \end{wrapfigure}
% \begin{figure}
%     \centering
%     \includegraphics[width=0.5\linewidth]{procedure.pdf}
%     \caption{Graphical illustration for \name{ROAM}. $Q_{\mathrm{MM}}(s, a)$ is equal to $\operatorname{Median}(\{\widehat{Q}_k^\pi(s, a)\}_{k=1}^K)$.}
%     \label{fig:procedure}
% \end{figure}
% \vspace*{-15pt}

\begin{algorithm}
\caption{\underline{R}obust \underline{O}ff-policy Ev\underline{a}luation via \underline{M}edian-of means based Direct Method (\name{ROAM-DM})}\label{alg:mm-final}
\begin{algorithmic}[1]
\INPUT Policy $\pi$, data $\mathcal{D}$, data partitions number $K$, decay rate $\gamma$, base DM-type OPE algorithm \name{BaseOPE} 
\STATE Partition trajectories $\mathcal{D}$ into $K$ disjoint parts:  $\mathcal{D}_1, \ldots, \mathcal{D}_K$.
\FOR{$k=1, \ldots, K$}
\STATE $\widehat{Q}_k^\pi \leftarrow \name{BaseOPE}\left(\pi, \mathcal{D}_k, \gamma\right)$
\ENDFOR
\STATE $\widehat{J}^\pi \leftarrow \mathbb{E}_{s \sim \mathbb{G}, a \sim \pi(\bullet|s)} \text{Median}(\{\widehat{Q}_k^\pi(s, a)\}_{k=1}^K).$
\OUTPUT $\widehat{J}^\pi.$
\end{algorithmic}
\end{algorithm}

\textbf{Uncertainty quantification.} In many high-risk applications such as mobile health studies, in addition to a point estimate on a target policy’s value, it is crucial to quantify the uncertainty of the value estimates, which has attracted increasing attention in recent years \citep{dai2020coindice, shi2021deeply, liao2021off, liao2022batch, kallus2022doubly}. One prominent advantage of leveraging MM in OPE is that it is straightforward to produce uncertainty quantification. Specifically, with $\{\widehat{Q}_k^\pi\}_{k=1}^K$, we can have $K$ integrated value estimators as $\{\widehat{J}_k^\pi\}_{k=1}^K$. Notice that $\{\widehat{J}_k^\pi\}_{k=1}^K$ are \textit{i.i.d.}. When each $\widehat{J}_k^\pi$ is unbiased, %Then, 
$\mathcal{Q}_q(\{\widehat{J}_k^\pi\}_{k=1}^K)$ is a natural $1-q$ lower confidence bound for $\eta^\pi$, where $\mathcal{Q}_q(\cdot)$ returns the $q$-th lower quantile value among a set. In contrast, it is nontrivial to obtain uncertainty quantification with other robust estimators like the truncated mean.

\textbf{Variants.} Our proposal is general and has a few theoretical guaranteed variants. First, instead of applying the median operator to the Q-values in Step 5 in Algorithm~\ref{alg:mm-final}, we can apply the median operator to 
% (i) 
% the state-value $\widehat{V}_k^\pi(s)=\mathbb{E}_{a \sim \pi(\bullet|s)} \widehat{Q}_k^\pi(s, a)$ to obtain $\widehat{J}^\pi=\mathbb{E}_{s \sim \mathbb{G}} \operatorname{Median}(\{\widehat{V}_k^\pi(s)\}_{k=1}^K)$; 
% or (ii) 
the estimated integral value $\widehat{J}_k^\pi=\mathbb{E}_{s \sim \mathbb{G}, a \sim \pi(\bullet|s)} \widehat{Q}_k^\pi(s, a)$ to obtain $\widehat{J}^\pi=\operatorname{Median}(\{\mathbb{E}_{s \sim \mathbb{G}, a \sim \pi(\bullet|s)} \widehat{Q}_k^\pi(s, a)\}_{k=1}^K)$. We study this variant, called \name{ROAM-Variant}, by empirical studies.

% Second, notice that many DM-type algorithms are iterative, such as Fitted Q-Evaluation \citep[FQE,][]{le2019batch} and MRDR \citep{farajtabar2018more}. It is natural to consider applying MM to each iteration of these algorithms %so that the final estimates are also robustified. 
% to robustify the final estimator. 
% We formulate this idea for FQE in Algorithm~\ref{alg:fqe-mm-internal} (\name{ROAM-FQE}, see complete algorithm in Appendix~). Algorithm~\ref{alg:fqe-mm-internal} derives robust intermediate estimators by replacing the heavy-tailed target $Y=R+\gamma \mathbb{E}_{a \sim \pi(S^{\prime})} Q\left(S^{\prime}, a\right)$ with a MM-type target $Y=R+\gamma \mathbb{E}_{a \sim \pi (S^{\prime})} \operatorname{Median}(\{\widehat{Q}^\pi_{k}\left(S^{\prime}, a\right)\}_{k=1}^K)$. However, one issue is that, these estimators $\{\widehat{Q}_{k}^\pi\}_{k=1}^K$ (and all estimators after this update including the final ones) in Algorithm~\ref{alg:fqe-mm-internal} are not independent any longer. Therefore, it is not clear whether or not MM can still have theoretical benefits. Thus we only study this variant empirically.

% \textbf{Extension to IS estimators.} 
Next, we can extend the framework of MM to give a robust IS estimator. To illustrate this extension, we take the marginal important sampling (MIS) estimator \citep{liu2018breaking, xie2019towards} as our example. The MIS estimator first estimates the state-action density ratio $\omega^\pi(s, a):=d^\pi(s) \pi(a | s) / b(s, a)$ as $\widehat{\omega}^\pi(s, a)$. Here, $b(s, a)$ is the state-action density of behavior policy, $d^\pi(s)$ is the average visitation distribution, defined as $d^\pi(s)=$ $(1-\gamma) \sum_{t=0}^{\infty} \gamma^t d_t^\pi(s)$ where $d_t^\pi(s)$ is the distribution of state $s_t$ when we execute policy $\pi$. Then, the MIS estimates the value of $\pi$ as $\widehat{J}^\pi:=\mathbb{E}_{\mathcal{D}}\left[\widehat{\omega}^\pi(S, A) R\right]$. To apply the MM procedure, we partition $\mathcal{D}$ into $K$ disjoint parts $\mathcal{D}_1, \ldots, \mathcal{D}_K$; then for each $\mathcal{D}_k$, we estimate ratios $\widehat{\omega}_k^\pi(s, a)$ and compute $\widehat{J}_k^\pi:=\mathbb{E}_{\mathcal{D}_k}\left[\widehat{\omega}_k^\pi(S, A) R\right]$. Finally, we define the robust IS estimator as $\widehat{J}^\pi=\operatorname{Median}(\{\widehat{J}_k^\pi\}_{k=1}^K)$. We summarize our method in Algorithm~\ref{alg:roam-mis} in Appendix~\ref{sec:mis}, which we refer to as \name{ROAM-MIS}. Our method can be similarly extended to accommodate doubly robust methods \citep{thomas2016data, kallus2020double}.

% \begin{algorithm}
% \caption{\underline{R}obust \underline{O}ff-policy Ev\underline{a}luation via \underline{M}edian-of-means based \underline{F}itted \underline{Q}-\underline{E}valuation (\name{ROAM-FQE})}\label{alg:fqe-mm-internal}
% \begin{algorithmic}[1]
% \INPUT Policy $\pi$, Data $\mathcal{D}$, decay rate $\gamma$, number of iterations $M$, number of partitions $K$.
% \STATE Partition data $\mathcal{D}$ into $K$ disjoint parts: $\mathcal{D}_1, \ldots, \mathcal{D}_K$
% \STATE Initialize $K$ Q-functions with parameters $w_1, \ldots, w_K$
% \FOR {$m=1, \ldots, M$}
% \FOR {$k=1, \ldots, K$}
% \STATE $\quad$ For each $\left(S, A, R, S^{\prime}\right) \in \mathcal{D}_k$, compute:
% \begin{align*}
% Y \leftarrow R+\gamma \operatorname{Median}(\{\mathbb{E}_{a \sim \pi\left(S^{\prime}\right)} \widehat{Q}_{w_k}^\pi\left(S^{\prime}, a\right)\}_{k=1}^K)
% \end{align*}
% % \STATE  $\quad$ Update $Q_{w_k}:$
% % \begin{align*}
% % w_k \leftarrow \underset{w_k}{\arg \min } \mathbb{E}_{\mathcal{D}_k}(Y-\widehat{Q}_{w_k}^\pi(S, A))^2
% % \end{align*}
% \STATE  $\quad$ Update $Q_{w_k}:$
% $w_k \leftarrow \underset{w_k}{\arg \min } \mathbb{E}_{\mathcal{D}_k}(Y-\widehat{Q}_{w_k}^\pi(S, A))^2$
% \ENDFOR
% \ENDFOR
% \STATE $\widehat{J}^\pi \leftarrow \mathbb{E}_{s \sim \mathbb{G}, a \sim \pi(\bullet|s)} \operatorname{Median}(\{\widehat{Q}_{w_k}^\pi(s, a)\}_{k=1}^K)$
% \OUTPUT $\widehat{J}^\pi$
% \end{algorithmic}
% \end{algorithm}

\subsection{MM for OPO with Pessimism}\label{sec:mm-room}
In this section, we introduce the extension of our proposal to OPO, called Robust OPO via Median-of-means (\name{ROOM}). To illustrate, we focus on value-based OPO algorithms in this section. A value-based OPO algorithm typically first estimates the optimal Q-function as $\widehat{Q}^*$, and then derives the corresponding optimal policy as either $\widehat{\pi}^*(s)=\arg\max_{a} \widehat{Q}^*(s, a)$ or $\widehat{\pi}^*(s)= \arg\max_{\pi \in \Pi} \mathbb{E}_{s \sim \mathbb{G}, a \sim \pi(\bullet|s)} \widehat{Q}^*(s, a)$ (when a policy class $\Pi$ is prespecified). Popular methods include fitted Q-iteration \cite[FQI,][]{ernst2005tree}, LSTD Q-learning \cite[LSTD-Q,][]{lagoudakis2003least}, etc.

To design a robust value-based OPO algorithm, we can follow a similar procedure for OPE as in Section~\ref{sec:mm-roam}. Specifically, we can first split $\mathcal{D}$ into $K$ folds to estimate $K$ independent optimal Q-functions $\{\widehat{Q}_k^*\}_{k=1}^K$, then output a policy $\widehat{\pi}^*(s)=\arg \max _a \operatorname{Median}(\{\widehat{Q}_k^*(s, a)\})$. We present this algorithm in Algorithm~\ref{alg:fqi-mm-final}, which we call \name{ROOM} for Value-based Method (\name{ROOM-VM}).

\begin{algorithm}
\caption{\underline{R}obust \underline{O}P\underline{O} via \underline{M}edian-of-means for \underline{V}alue-based \underline{M}ethod (\name{ROOM-VM})}\label{alg:fqi-mm-final}
\begin{algorithmic}[1]
\INPUT Data $\mathcal{D}$, decay rate $\gamma$, number of data partitions $K$, base value-based OPO algorithm \texttt{BaseOPO}
\STATE Partition data $\mathcal{D}$ into $K$ parts: $\mathcal{D}_1, \ldots, \mathcal{D}_K$.
\FOR {$k=1, \ldots, K$}
\STATE $\widehat{Q}_k^* \leftarrow \texttt{BaseOPO}\left(\mathcal{D}_k, \gamma\right)$
\ENDFOR
\STATE $\widehat{\pi}^*(s) \leftarrow \arg\max\limits_{a} \operatorname{Median}(\{\widehat{Q}_k^*(s, a)\}_{k=1}^K)$ for any $s$.
\OUTPUT Policy $\widehat{\pi}^*$.
\end{algorithmic}
\end{algorithm}

\textbf{Pessimism for robust OPO.} Insufficient data coverage is known as a critical issue to offline RL \citep{levine2020offline,xu2023offline}. When some state-action pairs are less explored, the related value estimation tends to have high variance and hence classical RL algorithms may produce a sub-optimal policy. The pessimistic principle effectively mitigates this issue, by taking the uncertainty of the value function estimation into consideration; see Figure~\ref{fig:bandit} for an illustration.

Employing the pessimism principle relies crucially on the construction of the uncertain set for the value function. However, as pointed out by \citet{zhou2022optimizing}, it is often challenging to derive a credible lower bound in general, and the tuning is typically sensitive. The issue becomes more serious when there exist heavy-tailed rewards. One prominent advantage of our MM procedure is that the pessimism mechanism can be naturally and efficiently incorporated. Specifically, we can replace the Median($\cdot$) operator in the Step 5 of Algorithm~\ref{alg:fqi-mm-final} by:
\begin{align*}
\widehat{\pi}^*(s) \leftarrow \arg\max_a \mathcal{Q}_q(\{\widehat{Q}_k^*(s, a)\}_{k=1}^K) \text { for any } s .
\end{align*}
By choosing $q<0.5$, we naturally obtain a pessimistic estimation for addressing insufficient data coverage. Moreover, quantile optimization is known to be robust \citep{wang2018exponentially} against heavy-tailed rewards as well (notice that the median operator is just a special case of quantile operators). Therefore, by this design, one can expect our method can address both insufficient data coverage and heavy-tailed rewards.  

We conclude this section with the comparison the standard methods. Standard methods construct the lower confidence bound of $Q^*(s, a)$ rely on subtracting the sample standard deviation (multiplied by a factor) from the sample mean, both are obtained from bootstrapping or concentration inequalities (see, e.g., \citet{kumar2019stabilizing, kumar2020conservative, bai2022pessimistic}). However, as \citet{hall1990asymptotic} points out, non-parametric bootstrap of the sample mean for heavy-tailed variables may not lead to a Gaussian asymptotic distribution. 
% Moreover, due to the overlapping between bootstrapped samples, the bootstrapped sample standard deviation is biased. 
Therefore, pessimistic RL algorithms based on bootstrapping may not work in heavy-tailed environments. Similar challenges apply to concentration inequality-based methods, especially when the variance does not exist. 

% Analogous to Algorithm~\ref{alg:fqe-mm-internal}, for iterative OPO algorithms, we can apply MM in every iteration. Take FQI as an example, we replace the definition of $Y$ in the Step 5 of Algorithm~\ref{alg:fqe-mm-internal} by:
% % \begin{align*}
% % Y \leftarrow R+\gamma \operatorname{Median}(\{\max _a \widehat{Q}_{w_k}^*\left(S^{\prime}, a\right)\}_{k=1}^K),
% % \end{align*}
% $Y \leftarrow R+\gamma \operatorname{Median}(\{\max_a \widehat{Q}_{w_k}^*\left(S^{\prime}, a\right)\}_{k=1}^K)$,
% then we can obtain a robust FQI algorithm. We refer to the resulting algorithm as $\name{ROOM-FQI}$ and defer its details to Algorithm~\ref{alg:mm-fqi-internal} in Appendix~\ref{sec:room-fqi-algorithm}. Moreover, we also consider its pessimistic variant by using a quantile operator $\mathcal{Q}_q(\cdot)$ rather than the median operator (denoted as \name{P-ROOM-FQI}). Similar as \name{ROAM-FQE}, we mainly study two variants empirically. {\color{red}It's noteworthy that SAC-$N$ \citep{an2021uncertainty} bears a strong resemblance to \name{P-ROOM-FQI}, as it assesses uncertainty by taking the pointwise minimum (i.e., setting $q=0.0$) of multiple Q-functions trained on the entire dataset. Hence, it can be regarded as a heuristic implementation of our approach within the SAC framework.}

% \vspace*{-8pt}
\section{THEORY}\label{sec:theory}
In this section, we focus on deriving the statistical properties of \name{ROAM}, designed for OPE. Meanwhile, our analysis can be easily extended to obtain the upper error bound of the estimated Q-function via \name{ROOM-VM} for OPO.
% , and hence the regret bound of the subsequently estimated optimal policy. We omit these results to save space. 
We begin with a set of technical assumptions. 
%summary of our theoretical results. To establish these theories, we need the following assumptions. 

\textbf{Assumption 1} (Independent transitions). $\mathcal{D}$ contains $n$ i.i.d. copies of $(S,A,R,S')$.

\textbf{Assumption 2} (Heavy-tailed rewards) There exists some $\alpha\in (0,1]$ such that $\mathbb{E} [|R|^{1+\alpha}]<\infty$.  %th absolute moment of the reward is bounded. 

\textbf{Assumption 3} (Sequential overlap). $\omega^{\pi}$ is bounded away from 0. 

We make a few remarks. First, the independence condition in Assumption~1 is commonly imposed in the literature to simplify the theoretical analysis \citep[see e.g.,][]{sutton2008convergent,chen2019information,fan2020theoretical,uehara2020minimax}. It can be relaxed by imposing certain mixing conditions on the underlying MDP \citep{kallus2019efficiently,chen2022well,bhandari2021finite}. 
Second, as we have commented earlier, nearly most existing works require the rewards to be uniformly bounded. On the contrary, Assumption~2 requires a very mild moment condition. 
When $\alpha<1$, this assumption even allows the variance of the reward to be infinity.  
Therefore, it is commonly used in the robust learning for bandits/RL literature \citep{bubeck2013bandits,zhuang2021no}. Assumption~3 corresponds to the sequential overlap condition that is widely imposed in the OPE literature \citep[see e.g.,][]{kallus2020double,shi2021deeply}. It essentially requires the support of the stationary state-action distribution under the behavior policy to cover that of the discounted state-action visitation distribution under $\pi$. 

We first study the theoretical properties of \name{ROAM-MIS}. 

\begin{theorem}\label{thm:MIS}
    Assume Assumptions 1-3 hold. Then for any $\delta>0$, by setting $K=\lceil 8 \log (2 / \delta)\rceil$, we have with probability $1-\delta$ that $|\widehat{J}^{\pi}_{\textrm{MIS}}-J^{\pi}|$ is of the order of magnitude 
    \begin{eqnarray*}
        M^{(1+\alpha)}_{R} \Big[ \|\widehat{\omega}^{\pi}-\omega^{\pi}\|_{\infty}
        +\|\omega^{\pi}\|_{\infty}[\log (1 / \delta)]^{\frac{\alpha}{1+\alpha}} n^{-\frac{\alpha}{1+\alpha}}\Big],
    \end{eqnarray*}
    where $M^{(1+\alpha)}_{R} = (\mathbb{E} |R|^{1+\alpha})^{\frac{1}{1+\alpha}}$ and  $\ell_{\infty}$-norm of any function $\omega$ is defined as $\|\omega\|_\infty \coloneqq \sup_{x} |\omega(x)|$. 
\end{theorem}

According to Theorem \ref{thm:MIS}, the estimation error of the proposed \name{ROAM-MIS} estimator can be decomposed into the sum of two terms. The first term depends crucially on the estimation error of the MIS ratio. By definition, the MIS ratio is independent of the reward distribution. As such, existing solutions are applicable to compute $\widehat{\omega}^{\pi}$ to obtain a tight estimation error bound. The second term depends on $\delta$ only through $(\log(1/\delta))^{\frac{\alpha}{1+\alpha}}$, which demonstrates the advantage of MM. It also relies on $\|\omega^{\pi}\|_{\infty}$, which measures the 
degree of distribution shift due to the discrepancy between the behavior and target policies. Finally, notice that both terms are proportional to $(\mathbb{E} [|R|^{1+\alpha}])^{\frac{1}{1+\alpha}}$. 
Compared with Proposition~\ref{prop:mm-estimator} for the classical MM estimation, Theorem~\ref{thm:MIS} shows \name{ROAM-MIS} presents a much greater challenge. Although this may seem intuitive, our quantitative analysis reveals two key insights: (i) the error of the estimated ratio has an additive effect on the OPE error, and (ii) the moment of reward introduces an additional scaling effect on the OPE error.

We next study \name{ROAM-DM}. For illustration purposes, we focus on a particular \name{BaseOPE} algorithm, the LSTD algorithm \citep{bradtke1996linear,boyan1999least}. 
In particular, let $\phi(S, A)$ denote a set of uniformly bounded basis functions, we parameterize the Q-function $Q^{\pi}(s,a)\approx \phi^\top(s,a) \theta^*$ for some $\theta^*$ and estimate this parameter by solving the following estimating equation with respect to $\theta$,
\begin{eqnarray*}
    \mathbb{E}_{\data} \Big\{\sum_a \phi(S, A) [R+\phi^\top(S', a) \theta-\phi^\top(s,a) \theta]\Big\}=0.
\end{eqnarray*}
Denote the resulting estimator as $\widehat \theta$.
Let $\xi_\pi(S, A)=\sum_{a}\Mean [\pi(a|S')\phi(S', a)|S, A]$. We impose some additional assumptions. 

\textbf{Assumption 4} (Realizability) There exists some $\theta^*$ such that $Q^{\pi}(s,a)=\phi^\top(s,a) \theta^*$ for any $s$ and $a$.  

\textbf{Assumption 5} (Invertibility) The minimum eigenvalue of 
$[\Mean \phi(S, A)\phi^\top(s,a)-\gamma \Mean \xi_\pi(S, A)\xi_\pi^\top(S, A)]$, denoted by $\lambda_{\min}$, is strictly positive. 

We again, make some remarks. First, Assumption 4 is a widely-used condition in the OPO literature to simplify the theoretical analysis \citep{duan2020minimax, jiang2020minimax, min2021variance, zhan2022offline}. It can be relaxed by allowing the Q-function to be misspecified, i.e., $\inf_{\theta}\mathbb{E} |Q^{\pi}(s,a)-\phi^\top(s,a) \theta|^2>0$ \citep[see, e.g.,][]{chen2019information}. Second, Assumption 5 is commonly imposed in the literature \citep{luckett2020estimating,perdomo2022sharp,shi2022statistical}. It can be viewed as a version of the Bellman completeness assumption when specialized to linear models \citep{munos2008finite}. Moreover, according to Theorem 3 in \citet{perdomo2022sharp}, invertibility is a necessary condition for solving any OPE problem using a broad class of linear estimators such as LSTD.
%parameterizes the Q-function using linear function approximation and 

\begin{theorem}\label{thm:DM}
   Assume Assumptions 1, 2, 4 and 5 hold. Then for any $\delta>0$ such that $\lambda_{\min}\gg (n/\log \delta^{-1} )^{-1/2}$, by setting $K=\lceil 8 \log (2 / \delta)\rceil$, we have with probability $1-\delta$ that $|\widehat{J}^{\pi}_{\textrm{DM}}-J^{\pi}|$ is of the order of magnitude 
    \begin{eqnarray*}
        \lambda_{\min}^{-1}(\mathbb{E} |R|^{1+\alpha})^{\frac{1}{1+\alpha}}\left(\log (1/\delta)\right)^{\frac{\alpha}{1+\alpha}}n^{\frac{-\alpha}{1+\alpha}}.
    \end{eqnarray*}
\end{theorem}

Using similar arguments in the proof of Theorem \ref{thm:DM} (see Appendix \ref{sec:proofthmDM}), we can show that the estimation error of the standard LSTD estimator grows at a polynomial order of $\delta^{-1}$. This again demonstrates the advantage of our proposal. 
%of the order of magnitude $\lambda_{\min}^{-1}(\mathbb{E} |R|^{1+\alpha})^{\frac{1}{1+\alpha}}\delta^{-\alpha/(1+\alpha)} n^{-\alpha/(1+\alpha)}$.
Moreover, Theorem~\ref{thm:DM} shows that, $\lambda^{-1}_{\min}$ serves a unique factor when compared to the classical MM estimator. 
Like Theorem~\ref{thm:MIS}, the $(1+\alpha)$-moment of the reward have a scaling effect on the OPE error. And thus, Theorem~\ref{thm:DM} highlights the challenges of robustifying the Q function and quantitatively illustrates the crucial terms for controlling the error of MM-based LSTD. To the best of our knowledge, this has largely unexplored in literature. Lastly, from the proof of Theorem~\ref{thm:DM}, we can prove \name{ROAM-Variant} also possesses the same order of error bound.

Finally, the subsequent theorem shows that \name{ROOM} derives a ``robusified'' pointwise lower bound of $Q^*(s, a)$. For illustration purposes, we concentrate on a specific \name{BaseOPO} method, the LSTD-Q algorithm. 

\begin{theorem}\label{thm:pess-q}
    If Assumptions 1, 2, 4 hold and Assumption 5 holds for $\pi^*$, then for any $(s, a) \in \mathcal{S} \times \mathcal{A}$, the following event  
    \begin{align*}
        Q^*(s, a) - \lambda_{\min}^{-1}(\mathbb{E} |R|^{1+\alpha})^{\frac{1}{1+\alpha}}n^{\frac{-\alpha}{1+\alpha}}
        \geq \mathcal{Q}_{q}(\{\widehat{Q}_k(s, a)\}_{k=1}^K)
    \end{align*}
    hold with probability at least $1 - \exp(-2K(2q-1)^2)$.
\end{theorem}

Notably, the gap between $Q^*(s, a)$ and pessimistic Q function estimation is proportional to $\mathbb{E} |R|^{1+\alpha}$, which exists even when rewards are heavy-tailed. Therefore, compared to pessimistic methods based on subtracting standard deviation, our proposal provides a robust lower bound and addresses heavy-tailedness and data coverage simultaneously.

\section{EXPERIMENTS}\label{sec:experiments}
% The code of our proposal is available at \url{https://github.com/Mamba413/RobustORL}.
% In this section, we empirically study the performance of \name{ROAM}-type algorithms (Section \ref{sec:experiments-ope}) and \name{ROOM}-type algorithms (Section \ref{sec:experiments-opo}). 

% \subsection{Experiments for OPE}\label{sec:experiments-ope}
\textbf{Experiments for OPE.}\label{sec:experiments-ope}
We first describe the procedure to generate the offline dataset $\data$ with heavy-tailed rewards.
We first train a policy online under $\mathcal{M}$ using PPO \citep{schulman2017proximal} and denote it as $\pi^*$, the optimal policy. 
Let the behaviour policy be a $\epsilon$-greedy policy based on $\pi^*$, i.e., $\pi^b \coloneqq (1 - \epsilon)\pi^* + \epsilon \pi^r$ where $\pi^r$ is a random policy. We set $\epsilon=0.05$ in our experiments.
We use $\pi^b$ to interact with the environment to collect an offline dataset $\mathcal{D}'$ with 100 episodes. 
In $\mathcal{D}'$, we add \textit{i.i.d.} zero-mean heavy-tailed random variables $\nu_{\textup{df}}$ to the rewards and obtain the dataset $\data$. 
We set $\nu_{\textup{df}}$ as a scaled $t_{\textup{df}}$ random variable, i.e., $\nu_{\textup{df}} \coloneqq  \kappa \nu'_{\textup{df}}/\sigma^2$, in which $\nu'_{\textup{df}}$ comes from a $t_{\textup{df}}$ distribution and $\sigma$ is the standard deviation of $\alpha$-truncated $t_{\textup{df}}$ random variable where $\alpha = 0.02$. 
The degree of freedom (df) controls the degree of heavy-tailedness, and $\kappa$ controls the impact of heavy-tailed noises. 
% To control the impact of the heavy-tailed noises
% Repeat the second step $N$ times so as to get $N$ trajectories: $\tau_1, \ldots, \tau_N$, and combine them into an offline dataset $\data$.

With the offline data $\data$, we investigate the performance of \name{ROAM-DM}, \name{ROAM-MIS}, and \name{ROAM-Variant} on estimating the value of $\pi^\ast$ by comparing with the vanilla FQE algorithm on a classical OpenAI Gym tasks, Cartpole. To ablate the effect of computing $K$ functions, we also compare with mean-aggregation-based methods, named \name{MA-DM} and \name{MA-MIS}. Moreover, considering the truncated mean (TM) as a useful technique in robust statistics, we also take it into consideration by implementing TM upon MIS, denoted by \name{TM-MIS}. Finally, since the FQE algorithm iteratively performs temporal difference (TD) updates, it is natural to leverage the structure of the MDP by applying MM to the TD updates. We formulate this idea for the FQE in Algorithm~\ref{alg:fqe-mm-internal} in Appendix~\ref{sec:roam-fqe} and denote this algorithm as \name{ROAM-FQE}. For all methods, we use a linear model $\phi^\top(s, a)\theta$ to model  $Q^{\pi^*}(s, a)$, where $\phi(s, a)$ includes the main effects and two-order interactions of the feature vector $(s^\top, a^\top)^\top$, generated by the \name{PolynomialFeatures} method of \name{scikit-learn} \citep{pedregosa2011scikit}. 
% The input \name{BaseOPE} algorithm is from \citep{duan2020minimax}. 
% Out of fairness, \name{ROAM-DM} employs \name{FQE} as the \name{BaseOPE} algorithm and \name{ROAM-FQE} the same linear model in \name{FQE}. 

% We evaluate the performance of offline policy evaluation on the optimal policy $\pi^*$. 
% estimated 
% by first sampling $5000$  episodes of length 100 following $\pi^\ast$ and computing a Monte Carlo estimator $\eta^{\textup{MC}}$. 
Given the ground truth $J^{\pi^*}$ computed via Monte Carlo, the mean squared errors of all methods are reported in Figure~\ref{fig:ope-cartpole}, aggregated over 100 replicates in each case. From Figure~\ref{fig:ope-cartpole}, all methods' MSEs reasonably decrease as the degree of freedom increases. The vanilla FQE method is outperformed by our methods, due to that it is not robust to heavy-tailed noises. The differences between \name{FQE} and our methods diminish when df increases, which is reasonable. Although \name{TM-MIS} does outperform the vanilla MIS, it is generally surpassed by \name{ROAM-MIS}. Additionally, \name{ROAM-MIS}, \name{ROAM-DM} and \name{ROAM-Variant} have a tiny difference, while \name{ROAM-FQE} surpasses all of them. This implies using MM at each iteration of FQE largely mitigates the negative impact of heavy-tailed rewards such that the whole dataset can be fully utilized during iterations. Finally, the comparison between \name{ROAM-DM} (or \name{ROAM-MIS}) and \name{MA-DM} (or \name{MA-MIS}) reveals the mean-ensemble strategy cannot handle heavy-tailed rewards. 

\begin{figure}[htbp]
    \begin{subfigure}{1.0\linewidth}
        \centering
        \caption{}
        \includegraphics[width=0.895\linewidth]{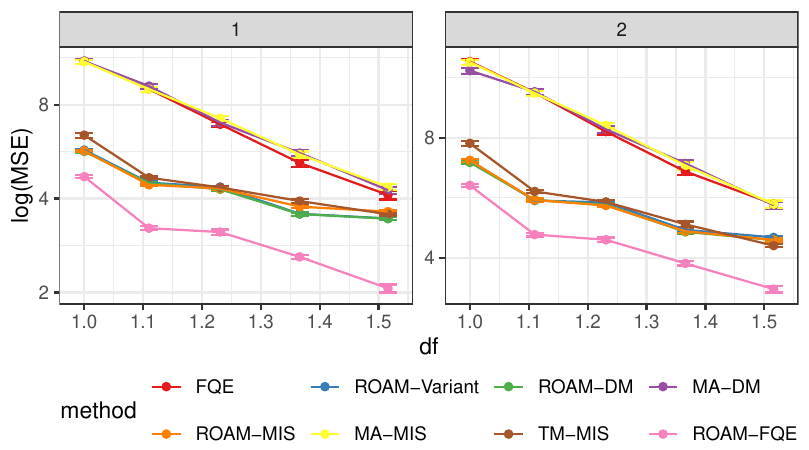}
        \label{fig:ope-cartpole}
    \end{subfigure}
    \hfill
    \begin{subfigure}{1.0\linewidth}
        \centering
        \caption{}
        \includegraphics[width=0.895\linewidth]{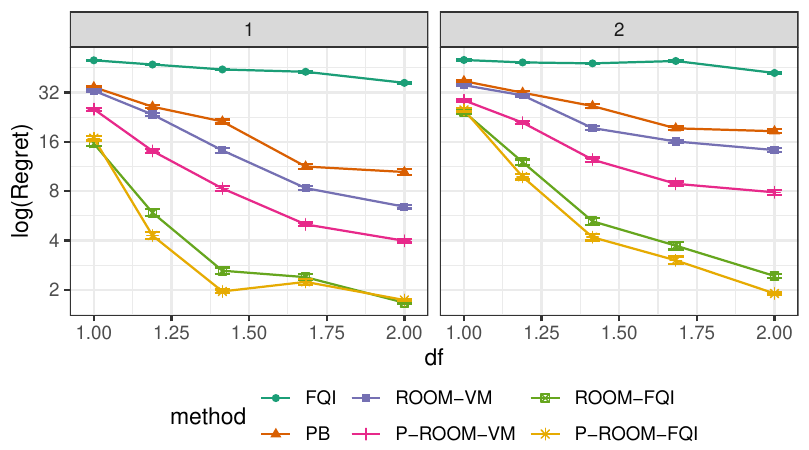}
        \label{fig:opo-cartpole}
    \end{subfigure}
    \caption{(a) OPE task: the trend of log(MSE) with the degree of freedom (DF). (b) OPO task: The trend of regret with respect to the DF. $\kappa$ takes value 1.0 (Left panel) and 2.0 (Right panel) in each subfigure. The error bar corresponds to 95\% CI. }
\end{figure}

\textbf{Compared with bootstrap-based \name{ROAM}.} Since our proposal can be interpreted as the ensemble of Q functions with a Median operator, another heuristic variant is using bootstrap instead of data partition. We conducted a comparison between this variant (denoted with the suffix ``\name{B-}'') and \name{ROAM}-type methods, adopting the same settings as in the previous section. The results are visualized in Figure~\ref{fig:boot_split_compare}. It is evident that the vertical gap between \name{ROAM-DM} and \name{B-ROAM-DM} is negligible. Furthermore, they perform significantly better than the corresponding \name{MA}-type methods. This phenomenon also holds for MIS-type methods presented in the right panel. Therefore, we can conclude that bootstrap serves as an alternative implementation for the proposed procedure.

\begin{figure}[htbp]
    \centering
    \includegraphics[width=1.0\linewidth]{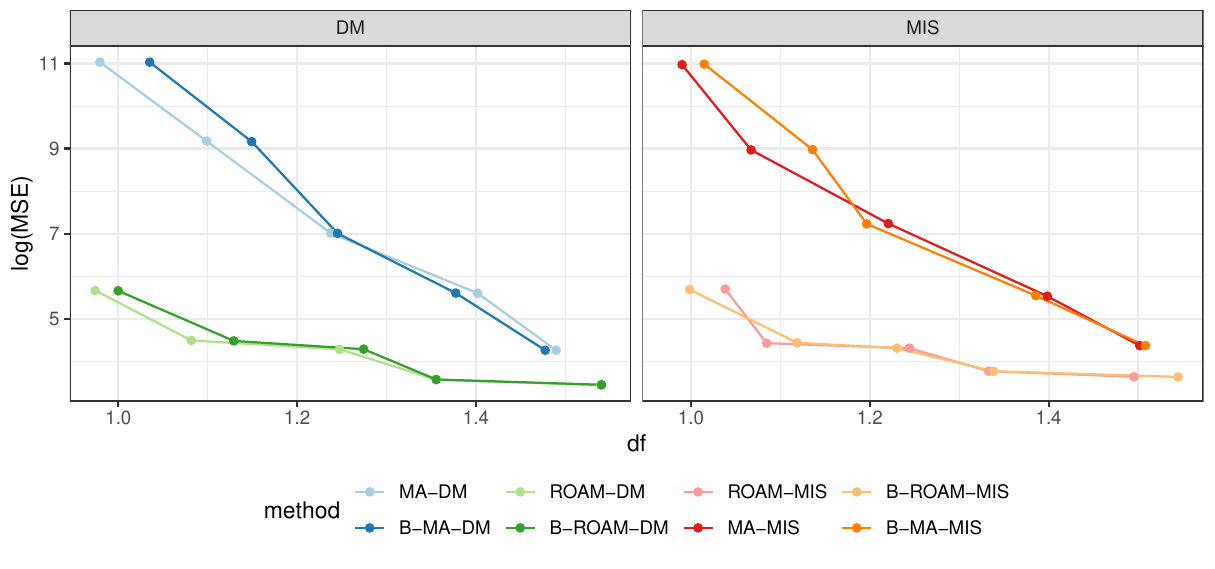}
    % \vspace{-15pt}
    \caption{The left panel presents the results for DM methods, and the right one displays the results for MIS methods. To prevent point overlap, random noise has been added to each point on the $x$-axis.}\label{fig:boot_split_compare}
\end{figure}
%Our methods have small differences; but, in general, \name{ROAM-FQE} seems to be the better one. 
%We note that the Q- and V-functions do not exist when $\textup{df} \leq 1$. We still test the empirical performance at this extreme case in which \name{ROAM-DM} and \name{ROAM-FQE} are superior to \name{FQE}.
% in the most of cases, the vanilla \name{ROAM-DM} is slightly better than its variant. 

% Then, we compute the mean square error (MSE) under $M$ replications:
% \begin{align*}
%     \textup{MSE}(\eta) = M^{-1}\sum_{m=1}^M (\widehat{\eta}_i - \eta^{\textup{MC}})^2.
% \end{align*}
% A small MSE means the offline policy evaluation method get a precise estimation. 

% \subsection{Experiments for OPO}\label{sec:experiments-opo}
% In this part, we first evaluate \name{ROOM-VM} and \name{ROOM-FQI} as well as their pessimistic variants in the Cartpole environment. Then, we turn to more complex environments.

% \subsubsection{Cartpole environment}\label{sec:experiments-opo-cartpole}
\textbf{Experiments for OPO (Cartpole environment).}\label{sec:experiments-opo-cartpole}
We study our algorithms: (i) \name{ROOM-VM} and (ii) \name{ROOM-VM} with pessimism (denoted as \name{P-ROOM-VM}), where \name{BaseOPO} are set as \name{FQI}. 
We compare two benchmark algorithms: \name{FQI} and the pessimistic-bootstrapping (\name{PB}) methods \citep{bai2022pessimistic}. 
See Appendix~\ref{sec:implement-opo-cartpole} for detailed implementations. 
Motivated by the powerful performance of \name{ROAM-FQE}, we also consider employing MM (and its pessimism version) in the TD update of \name{FQI}. We name these algorithms \name{ROOM-FQI} and \name{P-ROOM-FQI}, and defer their definitions to Algorithm~\ref{alg:mm-fqi-internal} in the Appendix.
% For the latter, we implement them upon Algorithms~\ref{alg:fqi-mm-final} and~\ref{alg:mm-fqi-internal} by following Remark~\ref{remark:pb-differ}.
% For fairness, we set $K = 5$ for all methods. 
We generate 400 episodes to form an offline dataset following the same procedure described in experiments for OPE. To evaluate the performance of a learned policy $\widehat{\pi}^*$, we compute its regret compared with the optimal policy. 
%%% the regret below needs update
% $\textup{Regret}(\widehat{\pi}^*) = \E_{s \sim \mathbb{G}}^{\pi^*} [\widehat{V}(s)] - \E^{\widehat{\pi}^*}_{s \sim \mathbb{G}} [\widehat{V}(s)]$.

% \begin{align*}
%     \textup{Regret}(\widehat{\pi}^*) = \E_{s \sim \Sspace_0}^{\pi^*} [\widehat{V}(s)] - \E^{\widehat{\pi}^*}_{s \sim \Sspace_0} [\widehat{V}(s)].
% \end{align*}
% A small regret means the $\widehat{\pi}^*$ is closed to $\pi^*$, and thus, $\widehat{\pi}^*$ is a superior policy. 

We report the numerical results of 100 replications in Figure~\ref{fig:opo-cartpole}. 
We see that the regret of \name{FQI} reasonably decreases when df goes up, but it decreases more slowly when $\kappa$ enlarges. 
We can also observe that \name{PB} improves over \name{FQI}, because it can properly address the insufficient data coverage issue in OPO. 
However, due to the existence of heavy-tailed rewards, \name{ROOM-VM} and \name{ROOM-FQI} can outperform \name{PB}. 
Even though we have no theoretical guarantee for \name{ROOM-FQI}, it shows a better numerical performance compared with \name{ROOM-VM}. 
Finally, we turn to \name{P-ROOM-VM} and \name{P-ROOM-FQI} in Figure~\ref{fig:opo-cartpole}. As expected, \name{P-ROOM-VM} (or \name{P-ROOM-FQI}) performs better than \name{ROOM-VM} (or \name{ROOM-FQI}) because it addresses the insufficient data coverage issue and the heavy-tailedness  simultaneously. 

\textbf{Experiments for OPO (D4RL datasets).}\label{sec:experiments-opo-d4rl}
We evaluate our proposed approach on the MuJoCo and Kitchen environments in the D4RL benchmarks \citep{fu2020d4rl}, which cover diverse dataset settings and domains. We generate heavy-tailed datasets by adding \textit{i.i.d.} noises into the reward observations, similar as the previous part. To show that the generality of our framework, in these datasets, we use another SOTA algorithm, sparsity Q-learning (\name{SQL}, \citet{xu2023offline}), as our \name{BaseOPO} algorithm. For the sake of fairness, we also take into account the mean aggregation (denoted as \name{MA}), which replaces the \texttt{Median} operator in Step 5 of Algorithm~\ref{alg:fqi-mm-final} with the \texttt{Mean} operator. Setting the discount factor $\gamma=0.99$, we train each algorithm for one million time steps and evaluate it every $5000$ time steps. Each evaluation consists of $10$ episodes. 
                        
We report the performance in Figure~\ref{fig:sql-results} and show learning curves in Figures~\ref{fig:sql-mm-mujoco} and~\ref{fig:sql-mm-kitchen}. It is worth noting that, in all cases, our methods are superior to the vanilla \name{SQL} algorithm. The superiority can be highly significant. For example, on the halfcheetah-expert dataset, \name{ROOM-VM} and \name{P-ROOM-VM} achieve a 30\% improvement over SQL, and in the kitchen environment, our proposal shows a 200\% improvement in \name{SQL}'s returns. This again shows that our proposal can effectively address the challenge of heavy-tailed rewards even in complex environments. The superiority of our proposal over \name{MA} reveals that the mean ensemble cannot handle heavy-tailed rewards but harm numerical performance because it has to use fewer samples to learn Q functions. Notably, in almost all cases, \name{P-ROOM-VM} surpasses \name{ROOM-VM} because \name{P-ROOM-VM} can more effectively address the issue of severe data insufficiency coverage. Furthermore, we also study the cases where \name{BaseOPO} is \name{IQL} \citep{kostrikov2022offline}. The results reported in Appendix~\ref{sec:experiments-iql} again show that our proposal has better performance than vanilla \name{IQL}, reflecting the versatility of our proposal. Finally, motivated by the success of the bootstrap-based variant and \name{P-ROOM-FQI}, we can further extend our proposal to the actor-critic paradigm and train agent with the whole dataset. Additionally, by setting $q=0.0$, this heuristic implementation leads to the exact \name{SAC-N} proposed by \citet{an2021uncertainty}, which is shown to be robust on heavy-tailed rewards in Figure~\ref{fig:sac-n} in the Appendix~\ref{sec:experiments-sacn}.

\begin{figure}[!t]
    \centering
    \includegraphics[width=\linewidth]{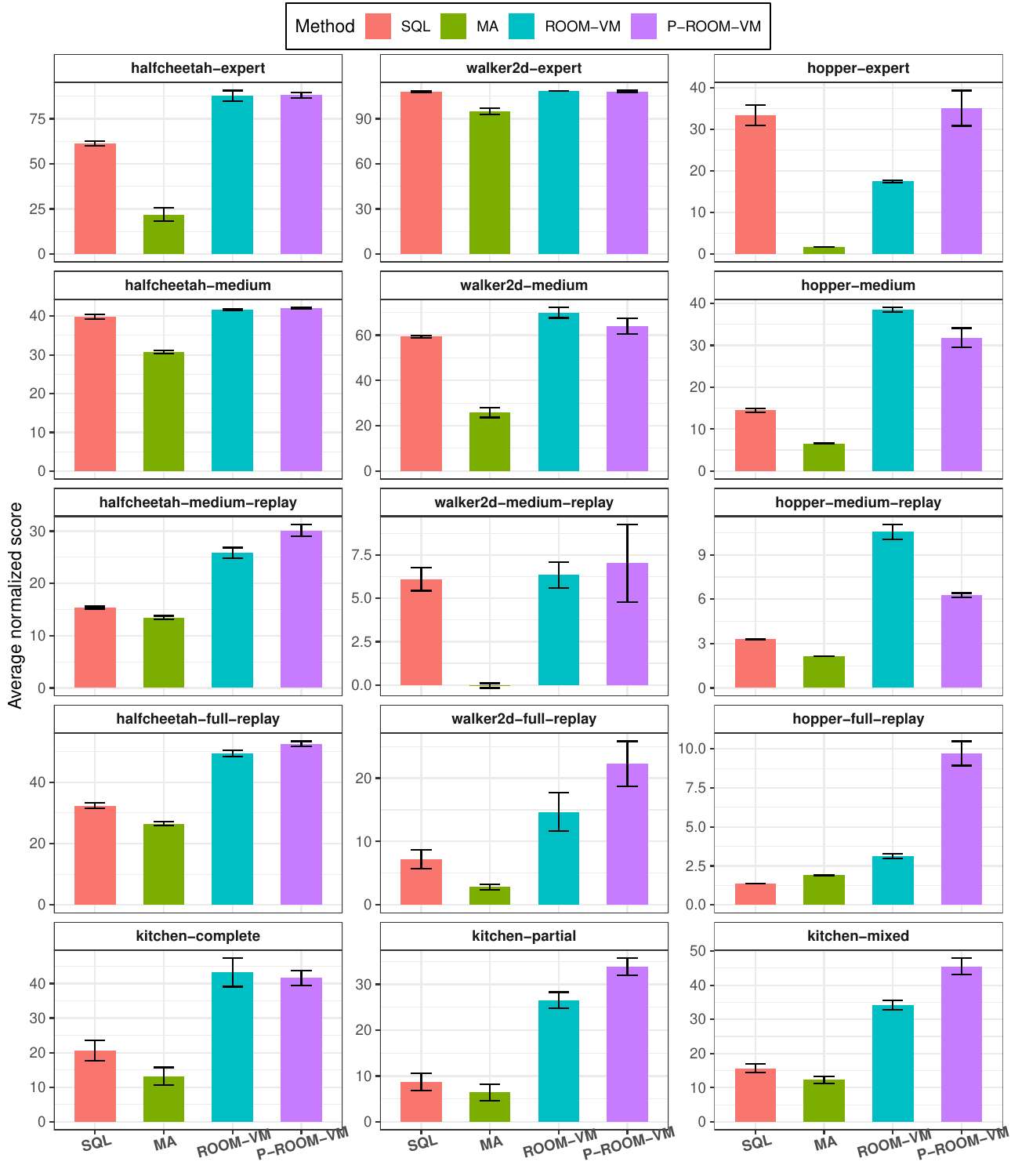}
    \caption{Results on D4RL datasets. Each bar corresponds to the average normalized score that is taken over the final 10 evaluations and 5 seeds. The error bar captures the 2 times standard error over 5 seeds.}
    \label{fig:sql-results}
\end{figure}

\textbf{Trade-off: Computation and robustness. }
We close this section with an examination of the trade-off between computation and robustness as $K$ varies. The results in Figure~\ref{fig:tradeoff_compute_robust} present the performance of \name{ROAM-FQE} in the Cartpole environment. As anticipated, the runtime of our proposal scales linearly with $K$. Yet, the computation is not demanding and can terminate in less than one second on a personal laptop. In terms of statistical performance, we observed that a moderate $K$ --- which well balances accuracy within each fold and accuracy of MM operator --- achieves the highest accuracy while requiring no more than half a second to terminate.

\begin{figure}[htbp]
    \centering
    \includegraphics[width=\linewidth]{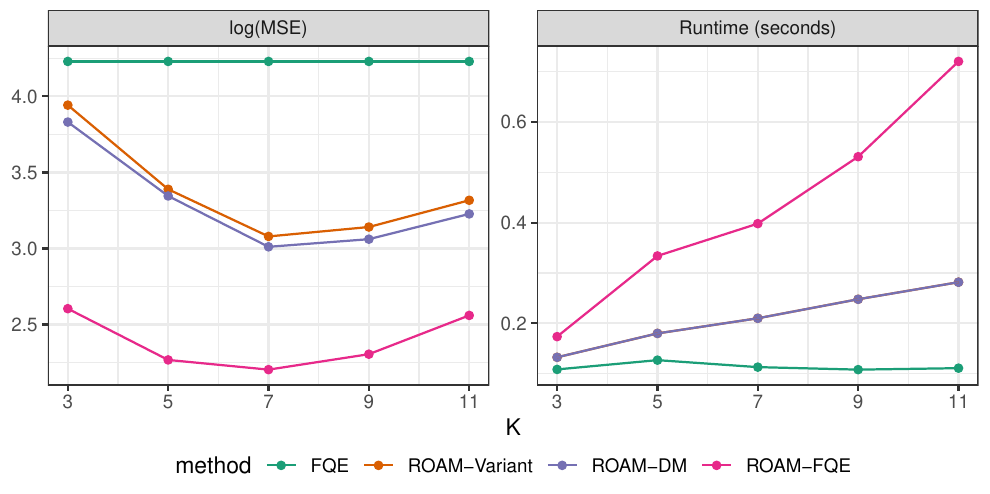}
    \caption{The $\log$(MSE) (left) and runtime (right) of Algorithm~\ref{alg:mm-final} as $K$ increases.}
    \label{fig:tradeoff_compute_robust}
\end{figure}

\section{CONCLUSIONS AND FUTURE WORKS}
Motivated by the real needs for robust offline RL methods against heavy-tailed rewards, we leverage the MM method in robust statistics to design a new frameworks that can robustify existing OPE and OPO algorithms. Our key insight is that employing MM to offline RL does more than just tackle heavy-tailed rewards--it offers valid uncertainty quantification to easily address insufficient coverage issue in offline RL as well. This insight is highly novel and, to our knowledge, has not been previously introduced in literature. Theoretical analysis demonstrates the advantages of our methods and extensive numerical studies support the empirical performance of our methods. An interesting future research direction is exploring the idea of leveraging the nature of heavy-tailed rewards and extending algorithms in \citet{ijcai2019p792} to the RL setting. 

\textbf{Acknowledgements} 

Jin Zhu and Runzhe Wan contribute equally to this article. Zhu’s and Shi’s research were supported by the EPSRC grant EP/W014971/1. The authors express their gratitude for the insightful feedback provided by the referees and the area chair, which significantly enhanced the initial version of this paper.

%Remind that, in numerical studies, our methods still show powerful performance on a extreme case that the mean of reward does not exist. In our future work, it would be interesting to develop a robust reinforcement learning framework to justify about it. It would be also worthwhile to develop methods to tackle the challenging that the transition of state including heavy-tailed observations.  

\bibliography{bib}

\newpage
\onecolumn
\aistatstitle{Robust Offline Reinforcement Learning with Heavy-Tailed Rewards:\\ Supplementary Materials}

\setcounter{section}{0}
\setcounter{figure}{0}
\setcounter{table}{0}
\renewcommand{\thesection}{A\arabic{section}}
\renewcommand{\thefigure}{A\arabic{figure}}
\renewcommand{\thetable}{A\arabic{table}}

\section{THEORETICAL PROOF}\label{sec:proof}
We use $c$ and $C$ to denote some general constants whose values are allowed to vary over time. Under Assumption~1, let $\{(S_i,A_i,R_i,S_i')\}_i$ denote the \textit{i.i.d.} transition tuples. 

\subsection{Proof of Theorem \ref{thm:MIS}}
\begin{proof}
Let $m=n/K$. The key to prove Theorem \ref{thm:MIS} is to show that for some properly chosen constant $c(\alpha)$ that depends only on $\alpha$ and 
\begin{eqnarray}\label{eqn:thm51proofeq1}
    \Delta\ge c(\alpha) (\mathbb{E} |R|^{1+\alpha})^{\frac{1}{1+\alpha}}\Big[\|\widehat{\omega}^{\pi}-\omega^{\pi}\|_{\infty}+\|\omega^{\pi}\|_{\infty}\Big(\frac{1}{m}\Big)^{\frac{\alpha}{1+\alpha}}\Big],
\end{eqnarray}
then
\begin{eqnarray}\label{eqn:thm51proofeq2}
    \mathbb{P}\Big(\Big|\frac{1}{m}\sum_{i=1}^m \widehat{\omega}^{\pi}(S_i,A_i) R_i-J^{\pi}\Big| \ge\Delta \Big)\le 0.2. 
\end{eqnarray}
The rest of the proof can similarly be established based on the arguments in the proof of Theorem 1 of \citet{lugosi2019mean} and we omit the details to save space. 

We focus on proving \eqref{eqn:thm51proofeq2} below. We begin by considering the following decomposition, 
\begin{eqnarray*}
    \frac{1}{m}\sum_{i=1}^m \widehat{\omega}^{\pi}(S_i,A_i) R_i-J^{\pi}=\frac{1}{m}\sum_{i=1}^m \omega^{\pi}(S_i,A_i) R_i-J^{\pi}+\frac{1}{m}\sum_{i=1}^m [\widehat{\omega}^{\pi}(S_i,A_i)-\omega^{\pi}(S_i,A_i)] R_i.
\end{eqnarray*}
For the first term, under Assumptions 2 and 3, the $(1+\alpha)$th moment of $\omega^{\pi}(S,A) R$ is upper bounded by $\|\omega^{\pi}\|_{\infty} \Mean |R|^{1+\alpha}$. Using the results in \citet{bubeck2013bandits} and \citet{devroye2016sub}, we can show that there exists some constant $c>0$ that depends only $\alpha$ such that
\begin{eqnarray}\label{eqn:thm51proofeq3}
    \mathbb{P}\Big\{\Big|\frac{1}{m}\sum_{i=1}^m \omega^{\pi}(S_i,A_i) R_i-J^{\pi}\Big| \ge c (\mathbb{E} |R|^{1+\alpha})^{\frac{1}{1+\alpha}} \|\omega^{\pi}\|_{\infty}  \Big(\frac{1}{m}\Big)^{\frac{\alpha}{1+\alpha}}\Big\} \le 0.1. 
\end{eqnarray}
As for the second term, it is upper bounded by $\|\widehat{\omega}^{\pi}-\omega^{\pi}\|_{\infty} (m^{-1}\sum_{i=1}^m |R_i|)$. Consider $m^{-1}\sum_{i=1}^m |R_i|$. We decompose it into the sum of $m^{-1}\sum_{i=1}^m (|R_i|-\mathbb{E} |R|)$ and $\mathbb{E} |R|$. Similar to \eqref{eqn:thm51proofeq3}, we can show $m^{-1}\sum_{i=1}^m (|R_i|-\mathbb{E} |R|)$ satisfies the following, 
\begin{eqnarray}\label{eqn:thm51proofeq4}
    \mathbb{P}\Big\{\Big|\frac{1}{m}\sum_{i=1}^m |R_i|-\mathbb{E} |R| \Big| \ge c (\mathbb{E} |R|^{1+\alpha})^{\frac{1}{1+\alpha}} \Big(\frac{1}{m}\Big)^{\frac{\alpha}{1+\alpha}}\Big\} \le 0.1. 
\end{eqnarray}
In addition, according to H{\"o}lder's inequality, we have $\mathbb{E} |R|\le (\mathbb{E} |R|^{1+\alpha})^{\frac{1}{1+\alpha}}$. This together with \eqref{eqn:thm51proofeq4} implies that the second term can be upper bounded by $C\|\widehat{\omega}^{\pi}-\omega^{\pi}\|_{\infty}(\mathbb{E} |R|^{1+\alpha})^{\frac{1}{1+\alpha}}$ for some constant $C>0$, with probability at least $0.9$. Combining this together with \eqref{eqn:thm51proofeq1} and \eqref{eqn:thm51proofeq3} yields \eqref{eqn:thm51proofeq2}. 
\end{proof}

\subsection{Proof of Theorem \ref{thm:DM}}\label{sec:proofthmDM}
%Under the realizability assumption and the boundedness assumption on $\phi$, the estimation error is proportional to 
\begin{proof}
Similar to the proof of Theorem \ref{thm:MIS}, it suffices to show that each base OPE  estimator satisfies \eqref{eqn:thm51proofeq2} for any 
\begin{eqnarray*}
    \Delta\ge c(\alpha) \lambda_{\min}^{-1}(\mathbb{E} |R|^{1+\alpha})^{\frac{1}{1+\alpha}} \Big(\frac{1}{m}\Big)^{\frac{\alpha}{1+\alpha}}.
\end{eqnarray*}
Under the realizability assumption and the boundedness assumption on $\phi$, the estimation error of the base OPE estimator is of the same order of magnitude as that of the based LSTD estimator $\widehat{\theta}$. It suffices to show that each base $\widehat{\theta}$ satisfies \eqref{eqn:thm51proofeq2} for any 
\begin{eqnarray*}
    \Delta\ge c \lambda_{\min}^{-1}(\mathbb{E} |R|^{1+\alpha})^{\frac{1}{1+\alpha}} \Big(\frac{1}{m}\Big)^{\frac{\alpha}{1+\alpha}},
\end{eqnarray*}
where $c$ denotes some positive constant that depends only on $\alpha$. 

By definition, $\widehat{\theta}-\theta^*$ equals
\begin{eqnarray*}
    &\Big[ \frac{1}{m}\sum_{i=1}^m \phi(S_i, A_i)\{\phi(S_i, A_i)-\gamma \sum_{a}\pi(a|S_i')\phi(S_i', a)\}^{\top} \Big]^{-1} \\
    \times
    &\Big[ \frac{1}{m}\sum_{i=1}^m \phi(S_i, A_i)\{R_i+\gamma \sum_{a}\pi(a|S_i')Q^{\pi}(S_i',a)-Q^{\pi}(S_i,A_i)\} \Big].
\end{eqnarray*}
Under the matrix invertibility assumption, using similar arguments in the proof of Lemma 3 of \citet{shi2022statistical}, we can show that the $\ell_2$ norm of the matrix
\begin{eqnarray*}
    \Big[ \frac{1}{m}\sum_{i=1}^m \phi(S_i, A_i)\{\phi(S_i, A_i)-\gamma \sum_{a}\pi(a|S_i')\phi(S_i', a)\}^{\top} \Big]^{-1}
\end{eqnarray*}
can be upper bounded by $C\lambda_{\min}^{-1}$ with probability approaching $1$. As for the second term, using the results in \citet{bubeck2013bandits} and \citet{devroye2016sub}, we can show that its $\ell_2$ norm is upper bounded by $\mathbb{E} (|R|^{1+\alpha})^{\frac{1}{1+\alpha}} m^{-\frac{\alpha}{1+\alpha}}$ with probability at least $0.9$. Combining these results yield that 
\begin{eqnarray*}
    \mathbb{P}\Big( \|\widehat{\theta}-\theta^*\|_2 \ge c (\mathbb{E} |R|^{1+\alpha})^{\frac{1}{1+\alpha}}  \Big(\frac{1}{m}\Big)^{\frac{\alpha}{1+\alpha}}\Big) \le 0.2,
\end{eqnarray*}
for some constant $c>0$. The proof is hence completed. 
\end{proof}

\subsection{Proof of Theorem~\ref{thm:pess-q}}
We first the following Lemma that would be used in our proof.
\begin{lemma}[\citet{bubeck2013bandits}]\label{lemma:bubeck}
    Let $\epsilon \in (0, 1]$ and $X_1, \ldots, X_n$ be i.i.d. random variable with mean $\mu$ and $(1+\epsilon)$-th moment $\Mean|X - \mu|^{1+\epsilon} = v$. Suppose each fold has $N$ observations such that $n = NK$, then for each $l \in \{1, \ldots, K\}$, the sample mean $\widehat{\mu}_{l} = \frac{1}{N}\sum\limits_{t=(l-1)N+1}^{lN}X_t$ satisfies
    \begin{align*}
        \prob\left( |\mu - \widehat{\mu}_l| \geq J \right) \leq \frac{6v}{n^\epsilon J^{1+\epsilon}}
    \end{align*}
    where for any $J > 0$.
\end{lemma}

Next, we prove Lemma~\ref{lemma:quantile-LB}, once it holds, we can follow a very similar proof for Theorem~\ref{thm:DM} to obtain the conclusion in Theorem~\ref{thm:pess-q}. 
% Thus, we only present the proof of Lemma~\ref{lemma:quantile-LB}. 
\begin{lemma}\label{lemma:quantile-LB}
    Under the same notations and conditions in Lemma~\ref{lemma:bubeck}, then $\widehat{\mu}^\mathcal{Q}_{q}$, the $q$-th quantile of $\{\widehat{\mu}_1, \ldots, \widehat{\mu}_K\}$, satisfies
    \begin{align*}
        \prob \left( \widehat{\mu} - (12v)^{\frac{1}{1+\epsilon}} \left(\frac{1}{N} \right)^{\frac{\epsilon}{1 + \epsilon}} > \widehat{\mu}_{q} \right) \geq 1 - \exp(-2K(2q-1)^2).
    \end{align*}
\end{lemma}
\begin{proof}
    According to Lemma~\ref{lemma:bubeck}, for each $l \in \{1, \ldots, K\}$, we have
    $$\prob\left(\mu - J \leq \widehat{\mu}_l\right) \geq \prob\left( |\widehat{\mu}_l - \mu| \leq J \right) \geq 1 - \frac{6v}{n^{\epsilon} J^{1+\epsilon}}.$$
    Let $Y_l = I(\mu - J \leq \widehat{\mu}_l)$ with $p \coloneqq \Mean(Y_l) \geq 1 - \frac{6v}{n^{\epsilon} J^{1+\epsilon}}$. Then, according to the definition of $\widehat{\mu}_{q}$, we have
    \begin{equation}\label{eqn:quantile-bound}
        \prob \left(\mu -J \leq \widehat{\mu}_q \right) 
        % &
        = \prob \left( \sum_{l=1}^K Y_l \geq qK \right) 
        % \\&
        \leq \exp\left(-2K(q - p)^2\right),
    \end{equation}
    where the second inequality comes from the Hoeffding inequality. 
    Note that for
    \begin{align*}
        J = \left(\frac{6v}{q} \right)^{\frac{1}{1+\epsilon}} \left(\frac{1}{N} \right)^{\frac{\epsilon}{1+\epsilon}},
    \end{align*}
    we can easily see that $p \geq 1 - q \geq q$ (due to $q \leq 0.5$), and thus, Equation~\eqref{eqn:quantile-bound} can be simplified to:
    \begin{align*}
        \prob \left(\mu -J \leq \widehat{\mu}_q \right) \leq \exp\left(-2K(2q - 1)^2\right).
    \end{align*}
    % This finally leads to the conclusion.
\end{proof}
% Upon on the proof of Lemma~\ref{lemma:quantile-LB}, we turn the proof for Theorem~\ref{thm:pess-q}.
% \begin{proof}
    % We divide the proof into two parts. In the first part, we show that, when LSPI converges, for the estimated optimal policy $\widehat{\pi}^{opt}$,  
    % \begin{align*}
    %     Q^{\widehat{\pi}^{opt}} - J \geq \mathcal{Q}_{q}
    % \end{align*}
% \end{proof}

\setcounter{algorithm}{2}
\section{ALGORITHM DETAILS}
\subsection{The ROAM-MIS Algorithm}\label{sec:mis}
\begin{algorithm}
    \caption{Robust Off-policy Evaluation via Median-of-means based Marginal Important Sampling (\name{ROAM-MIS})}\label{alg:roam-mis}
    \begin{algorithmic}[1]
        \INPUT Policy $\pi$, data $\mathcal{D}$, number of data partitions $K$, decay rate $\gamma$,  base marginal important sampling ratio estimation algorithm \name{BaseMIS}
        \STATE Partition trajectories $\mathcal{D}$ into $K$ parts: $\data_1, \ldots, \data_K$.
        \FOR{$k = 1, \dots, K$}
        \STATE $\widehat{\omega}^{\pi}_k \leftarrow \name{BaseMIS}(\pi, \mathcal{D}_k, \gamma)$ 
        \ENDFOR
        \STATE $\widehat{J}^\pi \leftarrow  \median ( \left\{ \mathbb{E}_{\data_k} [\widehat{\omega}^{\pi}_k(S, A) R] \right\}_{k=1}^K )$
        \OUTPUT $\widehat{J}^\pi$
    \end{algorithmic}
\end{algorithm}

\subsection{The ROOM-FQE Algorithm}\label{sec:roam-fqe}

Algorithm~\ref{alg:fqe-mm-internal} derives robust intermediate estimators by replacing the heavy-tailed target $Y=R+\gamma \mathbb{E}_{a \sim \pi(S^{\prime})} Q\left(S^{\prime}, a\right)$ with a MM-type target $Y=R+\gamma \mathbb{E}_{a \sim \pi (S^{\prime})} \operatorname{Median}(\{\widehat{Q}^\pi_{k}\left(S^{\prime}, a\right)\}_{k=1}^K)$. However, one issue is that, these estimators $\{\widehat{Q}_{k}^\pi\}_{k=1}^K$ (and all estimators after this update including the final ones) in Algorithm~\ref{alg:fqe-mm-internal} are not independent any longer. Therefore, it is not clear whether or not MM can still have theoretical benefits. Thus we only study this variant empirically.

\begin{algorithm}[H]
\caption{\underline{R}obust \underline{O}ff-policy Ev\underline{a}luation via \underline{M}edian-of-means based \underline{F}itted \underline{Q}-\underline{E}valuation (\name{ROAM-FQE})}\label{alg:fqe-mm-internal}
\begin{algorithmic}[1]
\INPUT Policy $\pi$, Data $\mathcal{D}$, decay rate $\gamma$, number of iterations $M$, number of partitions $K$.
\STATE Partition data $\mathcal{D}$ into $K$ disjoint parts: $\mathcal{D}_1, \ldots, \mathcal{D}_K$
\STATE Initialize $K$ Q-functions $\widehat{Q}^\pi_1, \ldots, \widehat{Q}^\pi_K$ with corresponding parameters $w_1, \ldots, w_K$
\FOR {$m=1, \ldots, M$}
\FOR {$k=1, \ldots, K$}
\STATE For each $\left(S, A, R, S^{\prime}\right) \in \mathcal{D}_k$, compute:
% \begin{align*}
% Y \leftarrow R+\gamma \operatorname{Median}(\{\mathbb{E}_{a \sim \pi\left(S^{\prime}\right)} \widehat{Q}_{k}^\pi\left(S^{\prime}, a\right)\}_{k=1}^K)
% \end{align*}
$Y \leftarrow R+\gamma \operatorname{Median}(\{\mathbb{E}_{a \sim \pi\left(S^{\prime}\right)} \widehat{Q}_{k}^\pi\left(S^{\prime}, a\right)\}_{k=1}^K)$
% \STATE  $\quad$ Update $Q_{w_k}:$
% \begin{align*}
% w_k \leftarrow \underset{w_k}{\arg \min } \mathbb{E}_{\mathcal{D}_k}(Y-\widehat{Q}_{w_k}^\pi(S, A))^2
% \end{align*}
\STATE  Update $\widehat{Q}^\pi_{k}$ by:
$w_k \leftarrow \underset{w_k}{\arg \min } \mathbb{E}_{\mathcal{D}_k}(Y-\widehat{Q}_{k}^\pi(S, A; w_k))^2$
\ENDFOR
\ENDFOR
\STATE $\widehat{J}^\pi \leftarrow \mathbb{E}_{s \sim \mathbb{G}, a \sim \pi(s)} \operatorname{Median}(\{\widehat{Q}_{k}^\pi(s, a)\}_{k=1}^K)$
\OUTPUT $\widehat{J}^\pi$
\end{algorithmic}
\end{algorithm}

\subsection{The ROOM-FQI Algorithm}\label{sec:room-fqi-algorithm}
Analogous to Algorithm~\ref{alg:fqe-mm-internal}, for iterative OPO algorithms, we can apply MM in every TD update. Take FQI as an example, we replace the definition of $Y$ in the Step 5 of Algorithm~\ref{alg:fqe-mm-internal} by:
% \begin{align*}
% Y \leftarrow R+\gamma \operatorname{Median}(\{\max _a \widehat{Q}_{w_k}^*\left(S^{\prime}, a\right)\}_{k=1}^K),
% \end{align*}
$Y \leftarrow R+\gamma \operatorname{Median}(\{\max_a \widehat{Q}_{k}^*\left(S^{\prime}, a\right)\}_{k=1}^K)$,
then we can obtain a robust FQI algorithm. 
% We refer to the resulting algorithm as $\name{ROOM-FQI}$ and defer its details to Algorithm~\ref{alg:mm-fqi-internal} in Appendix~\ref{sec:room-fqi-algorithm}. 
% Similar as \name{ROAM-FQE}, we mainly study two variants empirically. 
\begin{algorithm}[H]
\caption{\underline{R}obust \underline{O}ffline-Policy \underline{O}ptimization via \underline{M}edian-of-mean based \underline{F}itted \underline{Q}-\underline{I}teration (\name{ROOM-FQI})}\label{alg:mm-fqi-internal}
\begin{algorithmic}[1]
\INPUT Data $\mathcal{D}$, decay rate $\gamma$, number of iterations $M$, number of data partitions $K$.
% , uncertainty function $\Gamma^c(\cdot)$}
\STATE Partition data $\mathcal{D}$ into $K$ disjoint parts: $\data_1, \ldots, \data_K$. 
\STATE Initialize Q-functions $\widehat{Q}^*_1, \ldots, \widehat{Q}^*_K$ with parameters $w_1, \ldots, w_K$.
\FOR{$m = 1, \dots, M$}
\FOR{$k = 1, \dots, K$}
\STATE For each $\left(S, A, R, S^{\prime}\right) \in \mathcal{D}_k$, compute:
$Y= R + \gamma \max\limits_a \operatorname{Median}(\{\widehat{Q}^*_{k}(S', a)\}_{k=1}^K)$.
\STATE Update $Q_{k}$:
$w_k \leftarrow \argmin\limits_{w_k} \mathbb{E}_{\mathcal{D}_k}(Y - \widehat{Q}^*_{k}(S, A; w_k))^2$.
% \begin{align*}
%     w_k \leftarrow \argmin_{w_k} \sum_{(S, A, R, S') \in \mathcal{D}_k}(y - \widehat{Q}^\pi_{w_k}(S, A))^2,
% \end{align*}
\ENDFOR
% \STATE Sample $(S, A, R, S')$ from $\mathcal{D}$ and compute $y = \textup{Median}(\{Q_{w_k}(S, A)\}_{k=1}^K)$. 
% \STATE Update $Q_{w_0}$ via $w_0 \leftarrow w_0 + \left(y - Q_{w_0}(S, A)\right)\nabla_{w_0}Q_{w_0}(S, A).$
\ENDFOR
\STATE $\widehat{\pi}^*(s) \leftarrow \argmax\limits_a \textup{Median}(\{\widehat{Q}^*_{k}(s, a)\}_{k=1}^K)$.
\OUTPUT Policy $\widehat{\pi}^*$
\end{algorithmic}
\end{algorithm}	
Moreover, we also consider its pessimistic variant by using a quantile operator $\mathcal{Q}_q(\cdot)$ rather than the median operator in Step~5 of Algorithm~\ref{alg:mm-fqi-internal}. We denoted such a variant as \name{P-ROOM-FQI}.

\section{EXPERIMENTS: DETAILS}

% This rule of thumb also avoids the high computational burden caused by cross validation. Besides, this choice consumes fewer computational resources compared with the other large choice for $K$. 

\subsection{Settings for OPE}\label{sec:implement-ope-cartpole}
In the experiments for OPE (see Section~\ref{sec:experiments-ope}), we implement the minimax-optimal off-policy evaluation algorithm \citep{duan2020minimax} as the benchmarke \name{FQI} algorithm. 
Specifically, we use \name{Ridge} in \name{scikit-learn} with $\ell_2$-penalty fixed at 0.01. The implemented \name{FQI} algorithm serves as the \name{BaseOPE} algorithm for \name{ROAM-DM}. The implementation of \name{ROAM-FQE} also uses the same \name{Ridge} to update $Q_{w_k}$ in the Step 6 in Algorithm~\ref{alg:fqe-mm-internal}. The maximum number of iterations of all algorithms in Section~\ref{sec:experiments-ope} are fixed at 100. 

Next, we discuss the tuning of our algorithms. The only additional tuning parameter of \name{ROAM}-type algorithms is the number of partitions $K$, compared with its corresponding base algorithm. In our experiments, fixing $K=5$ already provides the desired performance and maintains a high computational efficiency. 
% In Algorithms~\ref{alg:mm-final} and~\ref{alg:fqe-mm-internal}, $K$ is the number of partitions for  dataset $\data$, which is an important hyperparameter. 
% \red{Even for the MM estimator of the population mean, the optimal value of $K$ depends on the data size and underlying distribution $F$, where the explicit expression }  
% \red{In theory, the optimal value of $K$ depends on the size of offline data and the underlying MDP, and it has no closed-form expression in general. Consequently, it is hard to decide. }
% One approach for approximating the optimal one is by cross validation although it may be too computationally intensive. 
In Appendix~\ref{sec:select-K}, we try a range of values for $K$ and find that our algorithms are  insensitive to this tuning parameter. 
% Our ablation studies for $K$ show that, we choose .. in most cases,  can lead to fairly good performance. 
One may choose this parameter via an adaptive method \citep{lugosi2019mean} as well.  

\subsection{Settings for OPO}
\subsubsection{Cartpole environment}\label{sec:implement-opo-cartpole}
For the experiments for OPO at Section~\ref{sec:experiments-opo-cartpole}, we implement the ridge-regression-based \name{FQI} algorithm as the benchmarked algorithm and the \name{BaseOPO} algorithm for \name{ROOM-VM}.
The \name{FQI} uses \name{Ridge} in \name{scikit-learn} to solve the optimal Q-function. We implement the \name{ROOM-FQI} with the same ridge regression with the same settings. For pessimistic variant of \name{ROOM}-type algorithms, we need an additional tuning parameter $q$, i.e., the quantile level. As argued in \citet{zhou2022optimizing}, the fact that one quantile \textit{explicitly} corresponds to one confidence level makes the tuning much easier than most existing methods where the relationship between the pessimism parameters and the confidence level is implicit and unknown. According to empirical results in Appendix~\ref{sec:selection-pess}, we find $q \in [0.1, 0.4]$ perform fairly well.  
We fix $q=0.1$ in our experiments.

We also implement a pessimistic-bootstapping OPO method, \name{PB}, to give a more comprehensive comparison. It is the same as \name{ROOM-VM} except the two modifications:
\begin{itemize}
    \item the Step 1 in Algorithm~\ref{alg:fqi-mm-final} is changed to: ``Sample $K$ bootstrapped samples from $\data$: $\data_1, \ldots, \data_K$'';
    \item the Step 5 in Algorithm~\ref{alg:fqi-mm-final} is modified to: 
    $$\widehat{\pi}(s) \leftarrow \argmax\limits_{a} \left[ \textup{Mean}(\{\widehat{Q}^*_k(s, a)\}_{k=1}^K) - 2\times \textup{Std}(\{\widehat{Q}^*_k(s, a)\}_{k=1}^K) \right] \;\textup{for any}\; s.$$
\end{itemize}
% The first version, called \name{PB}, is the same as \name{ROOM-VM} except the following two line replacements:
% The second version, named as \name{PB-FQI}, adopts Algorithm~\ref{alg:mm-fqi-internal} and performs the same modifications above. 

% \subsubsection{D4RL benchmark}\label{sec:implement-opo-cartpole}

\subsubsection{Mujoco environments}
\textbf{Datasets.} All D4RL datasets \citep{fu2020d4rl} on MuJoCo environments in the experiments are the ``v2'' version. The datasets on the Kitchen environment are the ``v0'' version.

\textbf{Network architecture.} The implementations of \name{SQL} is based on an open-source implementations from GitHub\footnote{\url{https://github.com/gwthomas/IQL-PyTorch}}, which largely reproduce the results in \citet{kostrikov2022offline}. Following the same architecture in SQL, both the critic and value networks are two-layer multi-layer perceptron (MLP) with 256 hidden nodes and ReLU activations. We recruit a deterministic policy network whose architecture is the same as critic network. 

The implementation of \name{N-SAC} is upon a public Github repository for SAC\footnote{\url{https://github.com/pranz24/pytorch-soft-actor-critic}}. Our implementation completely adopt the same actor-critic architecture in \citet{an2021uncertainty}. Specifically, the critic network is a three-layer MLPs with 256 hidden nodes and ReLU activations. The policy in \name{SAC-N} is a Gaussian policy network which enables automatic entropy tuning. As for \name{SAC-N (STD)} to be introduced in Section~\ref{sec:experiments-sacn}, it inherits the same architecture and hyperparameters as \name{SAC-N}. 

% We implement \name{SQL} upon an open-source implementation. Notably, the simplicity of \name{ROOM-VM} requires minimal modifications of the existing implementation. 

\textbf{Hyperparameters.} For the behavior-regularized term $\alpha$ in \name{SQL}, we set $\alpha = 10$ since Table~3 in \citet{xu2023offline} reports $\alpha=10$ leads to the best average result in MuJoCo environment. For the remaining parameters in \name{SQL}, we set them as their default hyperparameters, which are listed in Table~\ref{tab:hyperparameters}. Notice that, once we complete training \name{ROOM-VM}, the learned critics can be reused for \name{MA} and \name{P-ROOM-VM}. We recruit this programming trick to reduce the time for experiments. 
\begin{table}[htbp]
    \centering
    \caption{The hyperparameters of \name{SQL} used in the experiments for D4RL tasks.}
    \begin{tabular}{c|cc}
        \toprule
        & Hyperparameter & Value \\
        \midrule
        \multirow{7}{*}{\name{SQL}} 
        & Optimizer & Adam \citep{kingma2014adam} \\
        & Value/Critic learning rate & $3 \time 10^{-4}$ \\
        & Actor learning rate & Cosine schedule \citep{loshchilov2017sgdr} \\ 
        & Critic target update rate & $5 \times 10^{-3}$ \\
        & Mini-batch size & 256 \\
        & behavior-regularized $\alpha$ & 10 \\
        \midrule
        \name{ROOM-VM} \& \name{MA}   & Data partition $K$ & 5 \\
        \midrule
        \name{P-ROOM-VM} & Quantile $q$ & 0.0 \\
        \bottomrule
    \end{tabular}
    \label{tab:hyperparameters}
\end{table}

We summarized the hyperparameters for train \name{SAC-N} in Table~\ref{tab:nsac-hyperparameters}. 
\begin{table}[htbp]
    \centering
    \caption{The hyperparameters of \name{SAC-N} used in the experiments for D4RL tasks.}
    \begin{tabular}{cc}
        \toprule
        Hyperparameter & Value \\
        \midrule
        Critic/actor learning rate & $3 \time 10^{-4}$ \\
        Critic target update rate & $5 \times 10^{-3}$ \\
        Mini-batch size & 256 \\
        Ensemble number (a.k.a., $K$) & 10 \\
        Temperature     & 0.2 \\
        \bottomrule
    \end{tabular}
    \vspace{3pt}
    \label{tab:nsac-hyperparameters}
\end{table}

% \begin{figure}[H]
%     \centering
%     \includegraphics[width=\linewidth]{figure/iql_mm.pdf}
%     \vspace{-12pt}
%     \caption{Learning curves comparing the performance of \name{ROOM-VM} against \name{IQL} in the D4RL datasets with heavy-tailed noises on the rewards. Curves are averaged over 5 seeds and are smoothed by simple moving averages over three periods. }
%     \label{fig:sql-mm-mujoco}
% \end{figure}

\subsection{Computation Details}
\textbf{Hardware infrastructure.} The experiments in Carpole environment will finish in 5 hours on a personal laptop with 2.6 GHz 6-Core Intel Core i7 and 16 GB memory. 

As for the experiments for D4RL datasets, \name{ROOM-VM} generally consumes 14 hours to train a agent on a task with a machine with GPU Tesla P-100, while \name{SAC-N} roughly takes round 30 hours to train on the same device. 

\textbf{Time complexity analysis.} Let the computational cost of the base algorithm be $c(N)$ for a dataset with $N$ transition tuples. Our algorithm in general yields $\mathcal{O}(K \times c(N/K))$. For those based methods that have a linear computational cost in $N$ (e.g., FQE and FQI; see derivations in \citet{shi2021deeply}), our computational cost is at the same order. Moreover, our method is embarrassingly parallel.

\section{ADDITIONAL EXPERIMENTS AND RESULTS}\label{sec:addition-experiments}
\subsection{Learning Curves of \name{SQL}}\label{sec:lr-sql}
Figures~\ref{fig:sql-mm-mujoco} and~\ref{fig:sql-mm-kitchen} present the learning curves on MuJoCo and Kitchen environments, where the evaluations is conducted every 5000 training steps. Each curve is averaged over 5 seeds and is smoothed by simple moving averages over three periods. 
\begin{figure}[!t]
    \centering
    \includegraphics[width=\linewidth]{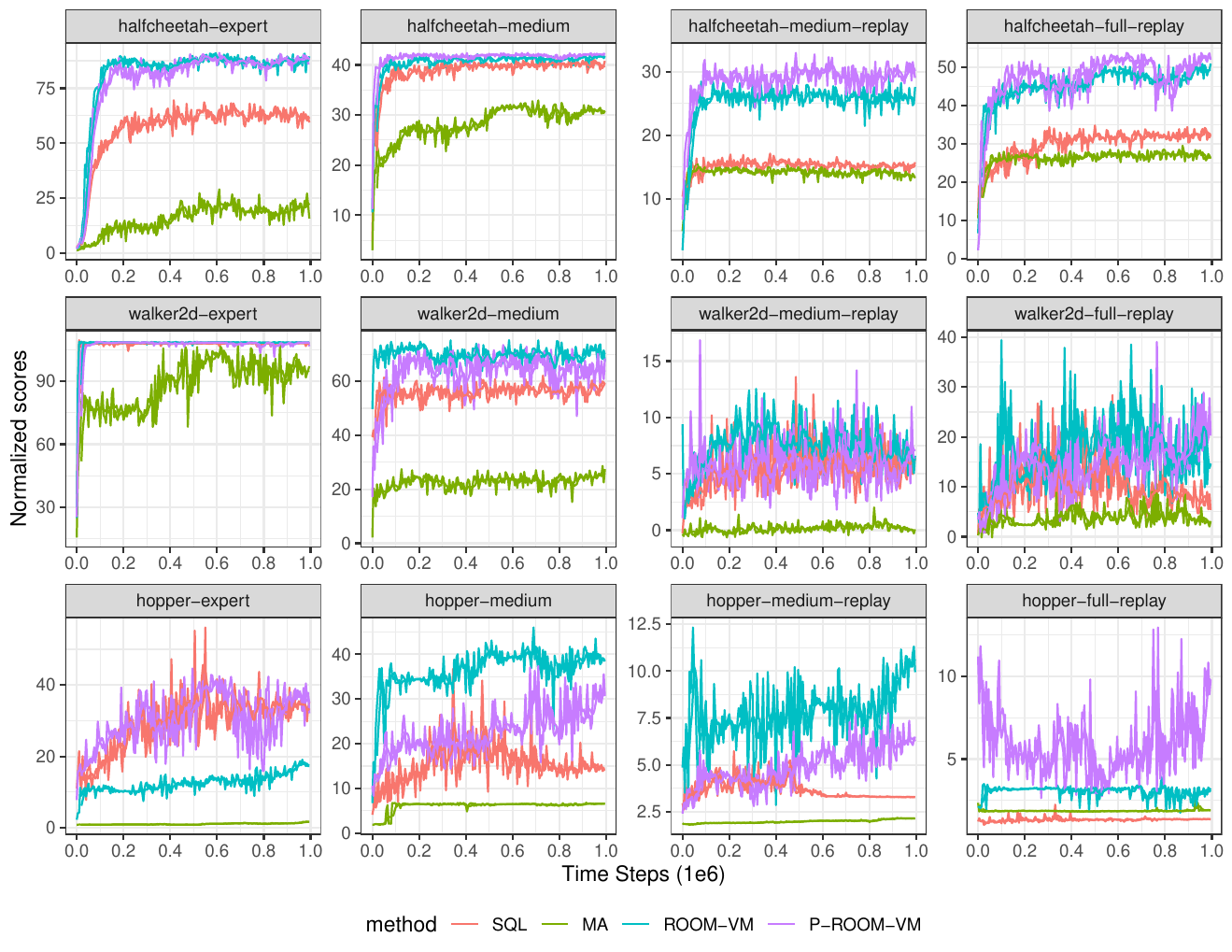}
    \caption{Learning curves of \name{SQL}, \name{MA}, \name{ROOM-VM}, \name{P-ROOM-VM} on D4RL MuJoCo locomotion datasets. }
    \label{fig:sql-mm-mujoco}
\end{figure}
\begin{figure}[!t]
    \centering
    \includegraphics[width=0.8\linewidth]{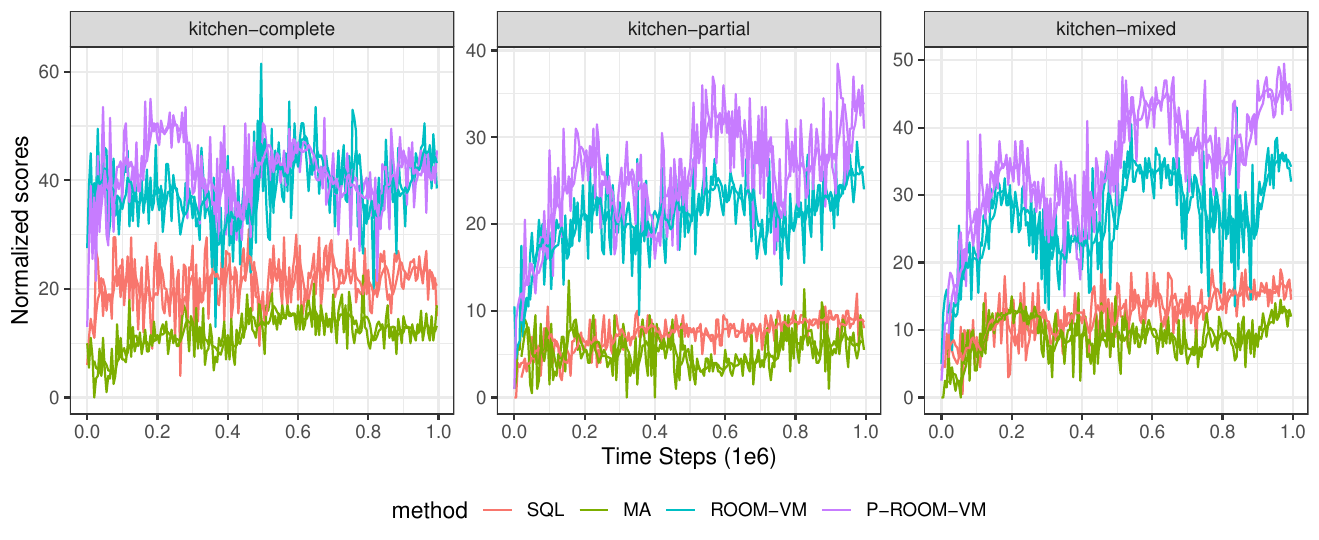}
    \caption{Learning curves of \name{SQL}, \name{MA}, \name{ROOM-VM}, \name{P-ROOM-VM} on D4RL Kitchen datasets. }
    \label{fig:sql-mm-kitchen}
\end{figure}

\subsection{Comparison with \name{IQL}}\label{sec:experiments-iql}
To show that our framework is general, in these datasets, we use another SOTA algorithm, implicit Q-learning (\name{IQL}, \citet{kostrikov2022offline}), as our \name{BaseOPO} algorithm.
We run IQL using the open-source implementation. Notably, the simplicity of ROOM-VM requires minimal modifications of the existing implementation. We adopt the data generation and model evaluation in Section~\ref{sec:experiments}. 

We report the performance in Table~\ref{tab:iql-results}. In almost all cases, our method is significantly superior to \name{IQL}. Notably, \name{ROOM-VM} can tackle the insufficient coverage challenge by employing \name{IQL} that can cope with insufficient coverage \citep{xu2023offline}. Moreover, the superiority of \name{ROOM-VM} and \name{P-ROOM-VM} are clearer in the walker and hopper environments, because these two environments are more challenging than others. For example, the returns of \name{ROOM-VM} and \name{P-ROOM-VM} on walker-medium are 20\% higher than that of \name{IQL}, and the returns of \name{ROOM-VM} and \name{P-ROOM-VM} on hopper-expert are about 200\% of  \name{IQL}'s returns. It is also worth noting that \name{P-ROOM-VM} shows comparable performance with \name{ROOM-VM} in most cases, and \name{P-ROOM-VM} has superior performance in the expert generating datasets because this setup has more severe data insufficient coverage issue. 
% \name{ROOM-VM} can tackle the insufficient coverage challenge by employing \name{SQL} that can cope with insufficient coverage \citep{xu2023offline}. 
% Moreover, the superiority of \name{ROOM-VM} and \name{P-ROOM-VM} are clearer in the walker and hopper environments, because these two environments are more challenging than others. 
% Finally, Figure~\ref{fig:sql-mm-mujoco} in the appendix shows that our methods generally converges more quickly than \name{IQL}.

\begin{table}[htbp]
    \centering
    \caption{Results for D4RL datasets. Each result is the division of average normalized score of \name{ROOM-VM} (or \name{P-ROOM-VM}) and \name{IQL}. We takes over the final 10 evaluations and 5 seeds. $\pm$ captures the 2 times standard error over 5 seeds.}
    \begin{tabular}{l|cc}
        \toprule
        Task Name	               & \name{ROOM-VM} 	  & \name{P-ROOM-VM}    \\
        \midrule
        halfcheetah-expert	       & 1.04 $\pm$ 0.02      & 1.03 $\pm$ 0.02     \\
    walker2d-expert	           & 1.02 $\pm$ 0.02      & 1.02 $\pm$ 0.02     \\
        hopper-expert	           & 1.89 $\pm$ 0.12      & 2.14 $\pm$ 0.13     \\
        \hline
        halfcheetah-medium	       & 0.99 $\pm$ 0.01      & 0.98 $\pm$ 0.01     \\
        walker2d-medium	           & 1.31 $\pm$ 0.08      & 1.40 $\pm$ 0.06     \\
        hopper-medium	           & 1.32 $\pm$ 0.03      & 1.28 $\pm$ 0.02     \\
        \hline
        halfcheetah-medium-replay  & 1.04 $\pm$ 0.04      & 1.02 $\pm$ 0.05     \\
        walker2d-medium-replay	   & 1.56 $\pm$ 0.26      & 0.95 $\pm$ 0.16     \\
        hopper-medium-replay	   & 1.15 $\pm$ 0.14      & 1.20 $\pm$ 0.11     \\
        \bottomrule
    \end{tabular}
    \label{tab:iql-results}
\end{table}
% \begin{wraptable}{r}{0.6\textwidth}
%     \centering
%     \begin{tabular}{l|cc}
%         \toprule
%         Task Name	               & \name{ROOM-VM} 	  & \name{P-ROOM-VM}    \\
%         \midrule
%         halfcheetah-expert	       & 1.04 $\pm$ 0.02      & 1.03 $\pm$ 0.02     \\
%     walker2d-expert	           & 1.02 $\pm$ 0.02      & 1.02 $\pm$ 0.02     \\
%         hopper-expert	           & 1.89 $\pm$ 0.12      & 2.14 $\pm$ 0.13     \\
%         \hline
%         halfcheetah-medium	       & 0.99 $\pm$ 0.01      & 0.98 $\pm$ 0.01     \\
%         walker2d-medium	           & 1.31 $\pm$ 0.08      & 1.40 $\pm$ 0.06     \\
%         hopper-medium	           & 1.32 $\pm$ 0.03      & 1.28 $\pm$ 0.02     \\
%         \hline
%         halfcheetah-medium-replay  & 1.04 $\pm$ 0.04      & 1.02 $\pm$ 0.05     \\
%         walker2d-medium-replay	   & 1.56 $\pm$ 0.26      & 0.95 $\pm$ 0.16     \\
%         hopper-medium-replay	   & 1.15 $\pm$ 0.14      & 1.20 $\pm$ 0.11     \\
%         \bottomrule
%     \end{tabular}
%     \caption{Results for D4RL datasets. Each result is the division of average normalized score of \name{ROOM-VM} (or \name{P-ROOM-VM}) and \name{IQL}. We takes over the final 10 evaluations and 5 seeds. $\pm$ captures the 2 times standard error over 5 seeds.}
%     \label{tab:mujoco-env}
% \end{wraptable}

\subsection{Robustness of \name{SAC-N}}\label{sec:experiments-sacn}

It's noteworthy that \name{SAC-N} \citep{an2021uncertainty} can be interpret to \name{P-ROOM-FQI}, as it assesses uncertainty by taking the pointwise minimum (i.e., setting $q=0.0$) of multiple Q-functions but trained on the entire dataset with an soft-actor-critic (SAC) paradigm. Hence, we can be regarded as a heuristic implementation of our approach, and we can expect that it enjoys robustness on heavy-tailed rewards. 

To illustrate, we implement \name{SAC-N} and compare with \name{SAC-N (STD)}, a method that achieves pessimistic estimation for Q function by pointwisely subtracting two times standard deviation of $N$ functions. To rephrase, \name{SAC-N (STD)} replaces the pointwise quantile $\mathcal{Q}_q(\{Q_k(s, a)\}_{k=1}^{N})$ with $\textup{Mean}(\{Q_k(s, a)\}_{k=1}^{N}) - 2 \times \textup{Std}(\{Q_k(s, a)\}_{k=1}^{N})$. We demonstrate the numerical performance \name{SAC-N} and \name{SAC-N (STD)} on halfcheetah-medium-v2 in Figure~\ref{fig:sac-n}. From the left panel of Figure~\ref{fig:sac-n}, we can see that the results of \name{SAC-N} are highly resembles to the results of Figure 1 in \citet{an2021uncertainty}. More importantly, despite \name{SAC-N} and \name{SAC-N (STD)} shares almost the same learning behavior when datasets has no heavy-tailed rewards (see left panel of Figure~\ref{fig:sac-n}), we can see that \name{SAC-N} is shown to be much robust to the heavy-tailed noises while \name{SAC-N (STD)} completely fails at this case. Therefore, it is also highly recommended to use \name{SAC-N} in environments with heavy-tailed rewards. 

\begin{figure}[htbp]
    \centering
    \includegraphics[width=0.8\linewidth]{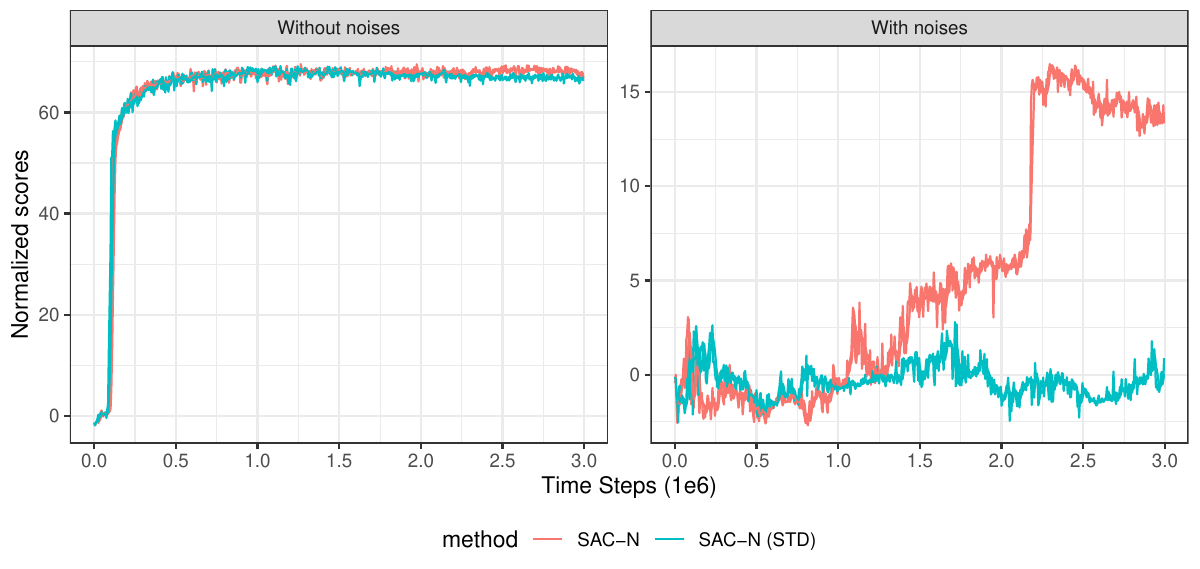}
    \caption{Learning curves comparing the performance of \name{SAC-N} against \name{SAC-N-STD} on the halfcheetah-medium-v2 dataset where the left panel corresponds to the case that heavy-tailed noises are not taken into consideration and the right one vice versa. Curves are averaged over 5 seeds and are smoothed by simple moving averages over three periods. }
    \label{fig:sac-n}
\end{figure}

\subsection{Selection of $K$}\label{sec:select-K}
In this part, we aim to study how the selection of $K$ influence the performance of \name{ROAM-DM}. Out of simplicity, we consider values for $K \in \{3, 4, \ldots, 10\}$, while the other settings adopt that in Section~\ref{sec:experiments-ope}. The estimation error of each algorithm is visualized in Figure~\ref{fig:selection-k}. 
From Figure~\ref{fig:selection-k}, we can see that, our methods exceeds \name{FQE} for all $K \in \{3, 4, \ldots, 10\}$. This implies that, whatever $K$ is taken, our methods are more suitable than \name{FQE} for offline data with heavy-tailed rewards. It is also worthy to note that, the optimal value of $K$ varies across algorithms and the degree of freedom of the heavy-tailed rewards. In terms of degree of freedom, since it is unknown in practice, there has no general criteria to decide the optimal~$K$. We suggest $K=5$ as a rule-of-thumb selection for all of our methods because this selection can result in a good performance. Notice that the comparison between ROAM-based methods and mean aggregation methods reveals taking the median operator is crucial for robustness --- the mean aggregation achieves a poor performance. 
% Moreover, for most of the methods, it seems that different choice for $K$ does not have a heavy impact on the regret. 

\begin{figure}[!t]
    \centering
    \includegraphics[width=0.52\linewidth]{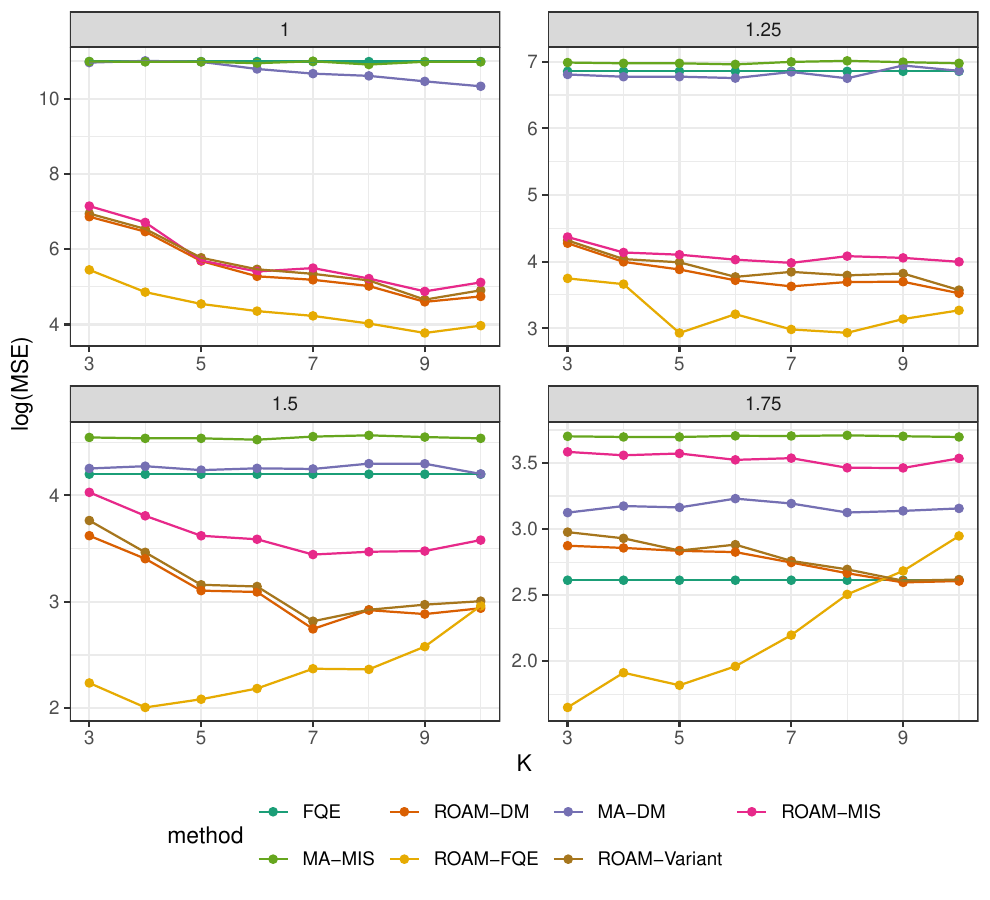}
    \caption{The Ablation on the tuning parameter $K$ for the OPE problem at the Cartpole environment. Each panel corresponds to a certain degree of freedom.}\label{fig:selection-k}
\end{figure}

\subsection{Selection of Pessimistic}\label{sec:selection-pess}
In this part, we aim to study how the selection of $q$ influence the performance of \name{ROAM-DM} and \name{ROOM-FQI}. We consider values for $q \in \{0.1, 0.2, 0.3, 0.4, 0.5\}$, while the other settings adopt that in the experiments on Cartpole environment. We visualize the regret of each algorithm in Figure~\ref{fig:selection-k}, from which we see that, our methods surpass \name{FQI} whatever the value of $q$. Besides, like the choice for $K$, both algorithms and the degree of freedom of the heavy-tailed rewards have an impact on the optimal value of $q$. Figure~\ref{fig:selection-pess} shows $q \in [0.1, 0.4]$ is a rule-of-thumb criterion for the guarantee of admirable numerical performance. 
\begin{figure}[!t]
    \centering
    \includegraphics[width=0.52\linewidth]{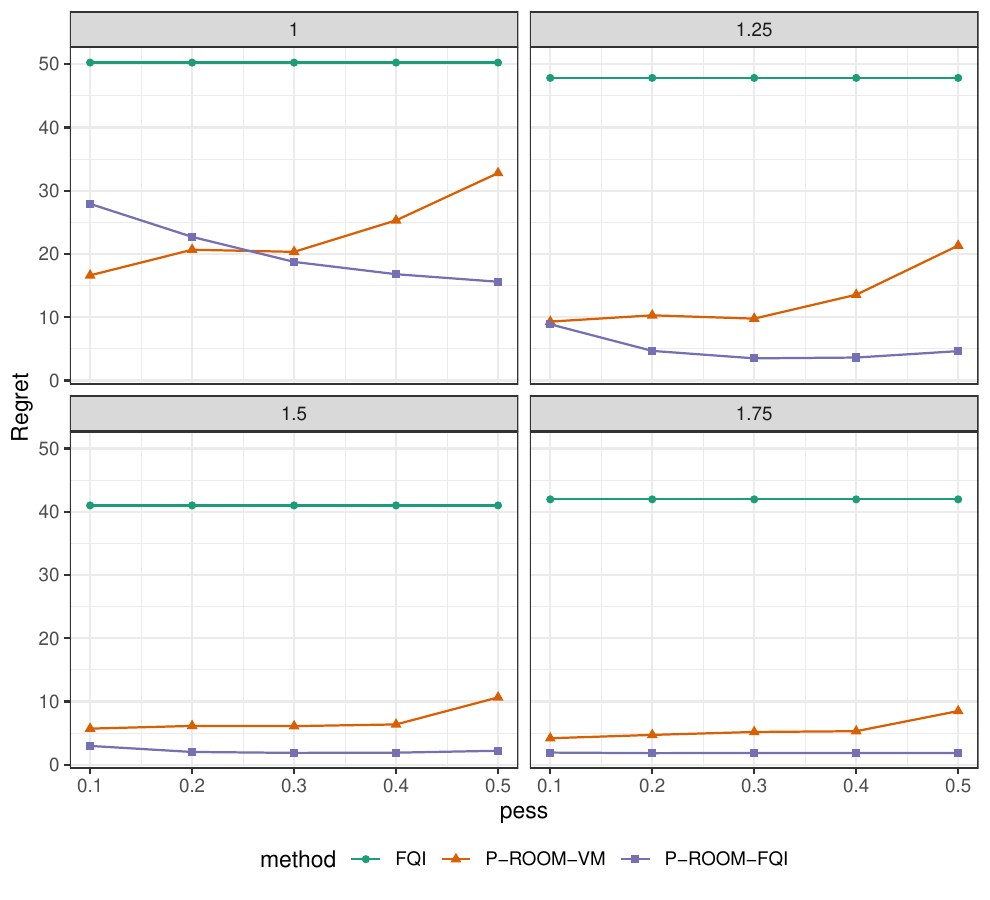}
    \caption{The ablation on the tuning parameter $q$ for the OPE problem at the Cartpole environment. Each panel corresponds to a certain degree of freedom.}\label{fig:selection-pess}
\end{figure}

\section{BROADER IMPACT STATEMENT}

Our approach provides offline RL methods to be applied to systems with heavy-tailed rewards. While our method can properly handle heavy-tailed rewards, it may also neglect the potential societal impact. For instance, heavy-tailed rewards in finance system may arise from fraudulent transactions or attacks on banking systems, which deserves adequate attention. One possible approach to monitor heavy-tailed rewards involves measures the gap between the two sides of Bellman equation. If the resulting value exhibits high variance, then the rewards warrants monitoring.

\end{document}